\documentclass[journal]{new-aiaa}

\usepackage[dvipsnames]{xcolor}

\makeatletter
\let\endfigure\undef
\def\endfigure{\end@float}
\let\endtable\undef
\def\endtable{\end@float}
\makeatother

\usepackage[]{appendix}
\usepackage[]{titlesec}
\usepackage[caption=false]{subfig}

\usepackage[]{graphicx}
\graphicspath{{graphics/}}

\usepackage{my_math}
\usepackage[]{physics}

\usepackage[]{tabularx}
\usepackage[]{multirow}
\usepackage[]{booktabs}
\usepackage[]{longtable}
\usepackage[]{threeparttable}

\usepackage[]{tikz,pgfplots}
\pgfplotsset{compat=1.18}
\usetikzlibrary{calc,positioning,arrows.meta,fit,backgrounds}

\usepackage[]{listings}
\usepackage{url}
\usepackage{hyperref}
\DeclareMathOperator{\st}{\mathrm{subject\ to\ }}

\title{Observability-Aware Control for Quadrotor Formation Flight with Range-only Measurement}
\author{%
    H S Helson Go\footnote{Ph.D. Candidate, Institute for Aerospace Studies, University of Toronto, 4925 Dufferin St., North York, Ontario, Canada, M3H 5T6, \href{mailto:hei.go@mail.utoronto.ca}{\texttt{hei.go@mail.utoronto.ca}}.},%
    Ching Lok Chong\footnote{Ph.D., Oxford University, OCIAM; currently independent researcher, \href{mailto:ching.lok.chong.secondary@gmail.com}{\texttt{ching.lok.chong.secondary@gmail.com}}},%
    Longhao Qian\footnote{Postdoctoral Fellow, Institute for Aerospace Studies, University of Toronto, 4925 Dufferin St., North York, Ontario, Canada, M3H 5T6, \href{mailto:longhao.qian@mail.utoronto.ca}{\texttt{longhao.qian@mail.utoronto.ca}}}, and Hugh H.T. Liu\footnote{Professor, AIAA Associate Fellow, Institute for Aerospace Studies, University of Toronto, 4925 Dufferin St., North York, Ontario, Canada, M3H 5T6, \href{mailto:hugh.liu@utoronto.ca}{\texttt{hugh.liu@utoronto.ca}}}.
}
\affil{University of Toronto Institute for Aerospace Studies, 4925 Dufferin St, North York, Canada, ON M3H 5T6}

\date{December 2023}
\tikzset{%
    cframe/.pic ={%
            \draw[thick, ->, pic actions] (0, 0)-- (0, -1 / 2);
            \draw[thick, ->, pic actions] (0, 0)-- (1 / 2, 0);
            \draw[thick, ->, pic actions] (0, 0)-- ({-sqrt (2) / 4}, {-sqrt (2) / 4});
            \fill[black] (0, 0) circle (0.05);
            \node[\tikzpictextoptions] at (0, 0) {\tikzpictext};
        }
}
\begin{document}

\maketitle

\begin{abstract}
    Cooperative Localization is a promising approach to achieving safe quadrotor formation flight through precise positioning via low-cost inter-drone sensors. This paper develops an observability-aware control principle tailored to quadrotor formation flight with range-only inter-drone measurements. The control principle is based on a novel approximation of the local observability Gramian (LOG), which we name the Short-Term Local Observability Gramian (STLOG). The validity of STLOG is established by proving its link to directional estimation precision in nonlinear systems.
    We propose the Observability Predictive Controller (OPC), a receding-horizon controller that generates optimal inputs to enhance information gain in weakly observable state directions by maximizing the minimum eigenvalue of the STLOG.\@
    This reduces the risk of estimator divergence due to the unbounded growth of uncertainty in weakly observed state components.
    Monte Carlo simulations and flight experiments are conducted with quadrotors in a GNSS-denied ferrying mission, showing that the OPC improves positioning confidence and estimator robustness.
\end{abstract}

\section*{Nomenclature}

\renewcommand\arraystretch{1.0}
\noindent\begin{longtable*}{@{}l @{\quad=\quad} l@{}}
    \(d\in \mathbb{R}^+\) & Distance, in meters, between the leader and the follower, measured by the follower \\
    \(f_l, f_f \in \mathbb{R}^{+}\) & Unitless normalized thrust delivered along each quadrotor's body z-axis \\
    \(\mathbf{H}\) & Observation matrix, which is the Jacobian of the observation model \\
    \(k_l, k_f\) & Thrust coefficients of the leader and follower quadrotors, respectively, in Newtons \\
    \(\check{\mathbf{P}}_k, \hat{\mathbf{P}}_k\) & Prior and posterior error covariance matrix at timestep \(k\) \\
    \(\mathbf{q}\in \mathcal{S}^3\)& Quaternion representing rotation of the follower relative to the leader \\
    \(\mathbf{r} \in\mathbb{R}^{3\times 1}\)& Position of the leader relative to the follower, expressed in the follower's frame \\
    \(T\) & Time horizon over which observability is evaluated \\
    \(T_k, T_{k+1}\) & K\textsuperscript{th} timestep under the OPC discretization scheme \\
    \(\boldsymbol{\Sigma}\) & Sensor noise covariance matrix, assumed to be time-invariant \\
    \(\mathbf{v}\in \mathbb{R}^{3\times 1}\)& Velocity of the leader relative to the follower, expressed in the follower's frame \\
    \(\delta t\) & Correlation timescale of the observation noise \\
    \(\Delta T\) & Discretization step size for the system dynamics and predictive control\\
    \(\phi_t\) & Flow of a dynamical system up to time \(t\) \\
    \(\boldsymbol{\Phi}_{t_0, t_1}\) & State transition matrix evaluated from \(t_0\) up to time \(t_1\) \\
    \(\boldsymbol{\omega}_l, \boldsymbol{\omega}_f \in \mathbb{R}^{3\times 1}\) & Body rate of each quadrotor, expressed in each vehicle's own frame \\
    \(\mathbf{1}, \mathbf{1}_i\) & Identity matrix and its \(i\)\textsuperscript{th} column (\emph{c.\ f.} the \(i\)\textsuperscript{th} standard basis vector) \\
    \(\mathbf{0}_{m \times n}\) & A \(m\)-times-\(n\) matrix filled with zeroes \\
\end{longtable*}

\section{Introduction}

A cooperative location system (CLS) allows multiple drones to work together to determine their positions more accurately than they could individually. It can complement cooperative guidance systems or controllers in applications such as path-following~\cite{xargay2013time}, conflict-free formation flight~\cite{yang2019distributed}, and GNSS-denied navigation as a source of position feedback.
A CLS measures inter-drone range or bearings~\cite{ding2023cooperative} and shares data between drones, improving the precision of positioning over short distances between drones in tight formations, or serving as a low-cost alternative to Global Navigation Satellite Systems (GNSS) when it becomes unreliable in dynamic and complex environments. However, a specific challenge arises in the range-only CLS of the leader-follower configuration: the follower drone(s) cannot determine their spatial position relative to the leader from a one-dimensional range measurement.
Therefore, cooperative motion strategies must be carefully redesigned to ensure that the follower drone can determine its full state from range observations, preventing unbounded growth of localization uncertainty.

Observability-aware trajectory optimization addresses this challenge by generating motion patterns to maximize the information gained from observations for precise state reconstruction.
In the guidance and navigation contexts, observability-aware motion strategies have been applied to beacon-based navigation and bearings-only pursuit scenarios~\cite{yu2015observability,anjaly2018observability,fujiwara2024observability} to improve state estimation performance. Optimizing measures of observability has been proven to contribute to reducing the covariance of EKF estimation~\cite{hinson2013observability} --- a classic goal of cooperative localization~\cite{trawny2004optimized,zhou2011optimized} investigated in our previous work~\cite{go2024trajectory}. The key to quantifying observability lies in the metrics of the LOG, which can be approximated numerically by the empirical local observability Gramian (ELOG)~\cite{krener2009measures}, or through the Lie derivatives of the observation model by the expanded empirical local observability Gramian (E\textsuperscript{2}LOG)~\cite{hausman2017observability,grebe2021observability}.
The latter approach enables observability-aware optimization to solve a critical issue with optimizing first-order metrics like the EKF covariance matrix, the Fisher information matrix (FIM)~\cite{fujiwara2024observability,anjaly2018observability}, and the Cram\'er-Rao Lower Bound (CRLB): if parts of the system state are indirectly observed, their corresponding columns in the observation matrix will be constantly zero.
This leads to rank-deficiency in the FIM, a part of the EKF covariance remaining unchanged from the initial prior covariance, and an undefined CRLB --- a failure to quantify information about indirectly observed states in every case.
A mitigation strategy --- integrating the FIM over a trajectory --- is still less effective than using the LOG to quantify trajectory-dependent information, since the time-integrated FIM only approaches the ELOG under ideal limits~\cite{powel2015empirical}.
The limitations of alternatives justify using the LOG to quantify the informativeness of motion-coupled observations, particularly in partially observed systems, such as range-only quadrotor CLS\@.

Despite the theoretical maturity of the LOG and the development of approximating frameworks, applying observability-aware optimization to quadrotor CLS remains challenging.
This is primarily due to the reliance on mechanically simplified models in existing work.
Solving motion design problems with such models may result in control strategies that are dynamically unfeasible for real quadrotors and unattainable observability metrics.
In range-based relative localization, Salaris~\cite{salaris2017online} applied observability-aware optimization based on metrics of the LOG to range-only landmark-based localization.
The constructibility Gramian (CG) was later proposed~\cite{salaris2019online} to improve information quantification, and De Carli~\cite{de2021online} integrated it into a perception-aware path planner to improve multi-agent positioning.
Lastly, observability-based optimization has been extended beyond offline trajectory design into the area of real-time optimal control~\cite{sharma2014observability,boyinine2022observability}.
However, these approaches rely on simplified dynamics, including planar quadrotors~\cite{salaris2019online}, first-order kinematic systems~\cite{de2021online}, or unicycles~\cite{sharma2014observability,boyinine2022observability}.
Optimizing observability using such simplified models introduces a modeling gap with real quadrotor dynamics and may result in motion that is dynamically infeasible and observability metrics that are unattainable in practice.

These limitations motivate our development of a novel control framework for range-only quadrotor CLS designed to improve system observability through controlled motion, while balancing model fidelity and computational complexity to enable online feedback-based control, and to provide a principled foundation for eventual real-time implementations as computational constraints are relaxed.

This paper presents an observability-aware control framework for quadrotors using a range-only CLS.
Its primary application is leader-follower navigation in GNSS-denied environments, where a leader vehicle with global navigation capability guides a follower that relies solely on inter-vehicle ranging.
Representative mission profiles include ferrying drones through GNSS-denied tunnels, collaborative exploration of caves, or search-and-rescue operations in urban canyons or dense forests.
In these scenarios, the drones maintain a dynamic formation while navigating to a common approximate destination.
The follower is not required to maintain a fixed position in the formation, except for staying within a safe communication and ranging distance, so it is free to maneuver to improve its state observability.

We introduce the concept of short-term local observability Gramian (STLOG), a novel point-wise approximation of the LOG that enables real-time optimization.
We then formulate an Observability-aware control principle that maximizes STLOG-based observability over the trajectory horizon.
Our main contributions address the challenge of performing real-time observability-aware control under realistic quadrotor dynamics and dimensionally-limited range-only sensing:
\begin{enumerate}
    \item We derive the STLOG and incorporate it into an observability-aware optimal control principle, aiming to improve state estimation along the least observable directions, preventing unbounded growth of localization uncertainty.
    \item We apply our controller to quadrotors in formation flight, demonstrating that observability-optimal trajectories improve positioning performance in Monte-Carlo simulations with realistic noise models, and flight experiments.
\end{enumerate}

The remainder of the paper is organized as follows.
In Section~\ref{sec:problemStatement}, we describe the range-only quadrotor CLS and the optimal control problem.
In Section~\ref{sec:observabilityAnalysis}, we study our CLS through classical observability analysis.
We introduce the STLOG in Section~\ref{sec:measurementOfObservability}, then incorporate it into our control principle in Section~\ref{sec:OAC}.
Section~\ref{sec:experimental} presents Monte-Carlo simulation and flight experimental results.
Finally, Section~\ref{sec:conclusion} concludes the paper and discusses future work.

\section{Problem Statement}\label{sec:problemStatement}

\subsection{Mathematical Preliminaries}
Lowercase \(a\) denotes scalars, bold lowercase \(\mathbf{a}\) vectors, and bold uppercase \(\mathbf{A}\) matrices.
\(\mathbf{a}_{i:j}\) denotes a slice of a vector containing elements-\(i\) up to \(j\).
The Euclidean norm is \(\norm{\cdot}\), while \(\norm{\cdot}_{A}\) denotes the norm induced by some space \(A\).
The Euclidean inner product is written using the matrix notation \(\mathbf{a}^\top\mathbf{b}\), while \(\langle \cdot,\cdot\rangle\) exclusively denotes inner products in the Hilbert space of functions.
The cross operator \(\times\) transforms a three-dimensional vector into a skew-symmetric matrix.
We adopt passive body-to-world quaternions: \(\otimes\) denotes the Hamiltonian quaternion product, \(\mathbf{R}(\cdot)\) converts a quaternion to a rotation matrix, and \(\mathcal{I}^\star(\cdot)\) converts a three-dimensional vector to a pure imaginary quaternion.
These operations are defined in Eq.~\eqref{eq:rotationOps}.
\begin{equation}
    \label{eq:rotationOps}
    \mathbf{a}^\times = \begin{bmatrix}
        0    & -a_3 & a_2  \\
        a_3  & 0    & -a_1 \\
        -a_2 & a_1  & 0
    \end{bmatrix}, \quad \mathbf{q} \otimes \mathbf{p} = \begin{bmatrix}
        q_4 \mathbf{p} + \mathbf{q}_{1:3}^\times & \mathbf{q}_{1:3} \\
        -\mathbf{q}_{1:3}                        & q_4
    \end{bmatrix} \begin{bmatrix}
        \mathbf{p}_{1:3} \\ p_4
    \end{bmatrix}, \quad\mathbf{R}(\mathbf{q}) = \mathbf{1} + 2q_4 \mathbf{q}_{1:3}^\times + 2 {\mathbf{q}_{1:3}^\times}^2, \quad  \mathcal{I}^\star(\mathbf{a}) = \begin{bmatrix}
        \mathbf{a} \\
        0
    \end{bmatrix}
\end{equation}

We analyze observability in the context of the continuous-time system in Eqs.~\eqref{eq:dynamicalSystem} and~\eqref{eq:observation}:
\begin{align}
    \dot{\mathbf{x}}(t) & = \mathbf{f}(\mathbf{x}(t), \mathbf{u}(t)), \quad \mathbf{x}(0) = \mathbf{x}_0 \label{eq:dynamicalSystem}
    \\
    \mathbf{y}(t)       & = \mathbf{h}(\mathbf{x}(t)) + \boldsymbol{\eta}(t),\label{eq:observation}
\end{align}
with state space \(\mathcal{X} \subset \mathbb{R}^{n\times 1}\), input space \(\mathcal{U} \subset \mathbb{R}^{m\times 1} \) and observation space \(\mathcal{Y} \subset \mathbb{R}^{p\times 1}\).
The deterministic flow \(\phi_t: \mathcal{X} \rightarrow \mathcal{X}\) describes the solution of the system under a fixed input trajectory $\mathbf{u}(t)$ with \(\mathbf{x}(t) = \phi_t(\mathbf{x}_0)\).%\footnote{While $\phi_t$ technically depends on time functionally through $\mathbf{u}(t)$, we omit this dependence unless needed
% }
The measurement noise \(\boldsymbol{\eta}(t)\) is modeled as a zero-mean Gaussian process (GP), with a covariance kernel $\mathbf{K}(t, t^\prime) = E\left[\boldsymbol{\eta}(t) \boldsymbol{\eta}{(t^\prime)}^\top \right]$.
% Its covariance function \(\mathbf{K}(t, t')\) and the associated covariance operator $\mathcal{K}$, acting on \(\mathbf{f}\in L^2([0,T], \mathbb{R}^p)\), are defined by:
% \begin{align}
%     \mathbf{K}(t,t') = E\left[\boldsymbol{\eta}(t)\boldsymbol{\eta}^\top(t')\right] \quad(\mathcal{K} \mathbf{f})(t) = \int_0^T \mathbf{K}(t,t') \mathbf{f}(t')  dt'.
% \end{align}

% \(\mathcal{K}\) is linear and self-adjoint on \(L^2([0,T], \mathbb{R}^p)\).
% The generalized trace of such operators with respect to any orthonormal bases \({\{e_n\}}_0^\infty\) is given by:
% \begin{equation}
%     \Tr\{\mathcal{K}\} \triangleq \sum_{n=0}^\infty \inner{\mathcal{K} e_n, e_n}.
% \end{equation}

% The operator $\mathcal{K}$ is said to be \emph{trace-class} if $\Tr\{\mathcal{K}\} < \infty$.

\subsection{The Quadrotor Relative Motion Model}

The configuration of a pair of cooperatively localizing quadrotor drones is shown in Fig.~\ref{fig:problemFormulation}.
This figure also illustrates the difference between a weakly-observable static formation and a dynamic formation to be commanded through optimizing measures of observability.

\begin{figure}[htpb]
    \centering

    \subfloat[Observability-aware control improves follower positioning.]{
        \input{problem_formulation_1.tikz}
    }\hspace{0.2cm}
    \subfloat[Uniform trajectory induces ambiguous positioning of the follower.]{
        \input{problem_formulation_2.tikz}
    }
    \caption{Example configurations of a CL-aided leader-follower quadrotor team.}%
    \label{fig:problemFormulation}
\end{figure}

$\mathbf{p}_l, \mathbf{p}_f \in \mathbb{R}^{3 \times 1}$ are the positions of the leader drone and the follower drone, respectively, expressed in the inertial frame.
\(\mathbf{q}_l, \mathbf{q}_f\in \mathcal{S}^3\) are the passive body-to-earth quaternions representing the orientations of the leader and the follower, respectively.
\(\mathbf{v}_l, \mathbf{v}_f\in \mathbb{R}^{3\times 1}\) are the velocities of the leader and the follower, respectively, expressed in the inertial frame.

We define the state components for the relative motion model of the fleet below.
\(\mathbf{r} = \mathbf{R}\left(\mathbf{q}_f\right)\left(\mathbf{p}_l - \mathbf{p}_f\right)  \in \mathbb{R}^{3 \times 1}\) is the relative position between the leader and the follower, expressed in the follower's body frame.
\(\mathbf{q} = {(\mathbf{q}_f)}^{-1} \otimes \mathbf{q}^{l} \in \mathcal{S}^3\) is the quaternion representing the rotation of the follower relative to the leader.
\(\mathbf{v} = \mathbf{R}\left(\mathbf{q}_f\right)\left(\mathbf{v}_l - \mathbf{v}_f\right) \in \mathbb{R}^{3 \times 1}\) is the relative velocity between the leader and the follower, expressed in the follower's body frame.
\(f_l, f_f \in \mathbb{R}^+\) are the normalized thrusts of the leader and the follower, respectively, and they are converted into real thrusts in Newtons through \(k_l, k_f \in \mathbb{R}^+\); \(\boldsymbol{\omega}_l, \boldsymbol{\omega}_f \in \mathbb{R}^{3\times 1}\) are the body rates of the leader and the follower, respectively expressed in their body frames.
The continuous-time dynamics of relative motion, playing the role of \(\mathbf{f}\) in Eq.~\eqref{eq:dynamicalSystem}, are:
\begin{align}
    \label{eq:dynamicsModel}
    \dot{\mathbf{x}} & = \mathbf{f}(\mathbf{x}, \mathbf{u}) \triangleq
    \begin{bmatrix}
        {\mathbf{r}}^\times\boldsymbol{\omega}_f + \mathbf{v}                                                                                         \\
        -\frac{1}{2} \mathcal{I}^\star(\boldsymbol{\omega}_f)\otimes\mathbf{q} + \frac{1}{2}\mathbf{q}\otimes\mathcal{I}^\star(\boldsymbol{\omega}_l) \\
        {\mathbf{v}}^\times\boldsymbol{\omega}_f + k_l f_l\mathbf{R}\left(\mathbf{q}\right)\mathbf{1}_3- k_f f_f\mathbf{1}_3
    \end{bmatrix}, \quad \mathbf{x} = \begin{bmatrix}\mathbf{r} \\ \mathbf{q} \\ \mathbf{v}\end{bmatrix}, \quad \mathbf{u} = \begin{bmatrix}
                                                                                                                                 \mathbf{u}_l \\
                                                                                                                                 \mathbf{u}_f \\
                                                                                                                             \end{bmatrix},
\end{align}
where $\mathbf{x} \in \mathbb{R}^{10 \times 1}$ is the system state and $\mathbf{u} \in \mathbb{R}^{8 \times 1}$ the input, in which \(\mathbf{u}_l = {\begin{bmatrix} f_l & {\boldsymbol{\omega}_l}^\top\end{bmatrix}}^\top \in \mathbb{R}^{4 \times 1}\), and $\mathbf{u}_f = {\begin{bmatrix}f_f & {\boldsymbol{\omega}_f}^\top\end{bmatrix}}^\top \in \mathbb{R}^{4 \times 1}$ are the inputs of the leader and the follower, respectively.
\(\mathbf{u}_l\) and \(\mathbf{u}_f\) are in the thrust and body rate format, a attitude-level setpoint commonly accepted by autopilot systems, and often seen in literature on optimal quadrotor control~\cite{falanga2018pampc} that assumes a downstream attitude controller tracks them sufficiently quickly.

The discrete-time dynamics based on Eq.~\eqref{eq:dynamicsModel} can be discretized by integration of RK4 over a timestep \(\Delta T\):
\begin{equation}
    \mathbf{x}_{k+1} = \mathbf{f}_d(\mathbf{x}_k, \mathbf{u}_{k}) \triangleq \mathbf{x}_k +  \texttt{RK4}(\mathbf{f}(\cdot, \mathbf{u}_k), \mathbf{x}_k, \Delta T),
\end{equation}

Equation~\eqref{eq:dynamicsModel} can be written in a control-affine form \(\dot{\mathbf{x}} = \mathbf{f}_0(\mathbf{x}) + \sum_{i=1}^m \mathbf{f}_i(\mathbf{x}) u_i\) as follows:
\begin{equation}
    \label{eq:dynamicsModelControlAffine}
    \dot{\mathbf{x}} = \underbrace{\begin{bmatrix}
            \mathbf{v} \\
            \mathbf{0} \\
            \mathbf{0}
        \end{bmatrix}}_{\mathbf{f}_0} + \underbrace{
        \begin{bmatrix}
            \mathbf{0} \\
            \mathbf{0} \\
            k_l\mathbf{R}(\mathbf{q})\mathbf{1}_3
        \end{bmatrix}}_{\mathbf{f}_1}f_l
    + \underbrace{\begin{bmatrix}
            \mathbf{0}                 \\
            \mathbf{J}^{+}(\mathbf{q}) \\
            \mathbf{0}
        \end{bmatrix}}_{\mathbf{f}_2}\boldsymbol{\omega}_l
    + \underbrace{
        \begin{bmatrix}
            \mathbf{0} \\
            \mathbf{0} \\
            -k_f\mathbf{1}_3
        \end{bmatrix}}_{\mathbf{f}_3}f_f + \underbrace{\begin{bmatrix}
            {\mathbf{r}}^\times         \\
            -\mathbf{J}^{-}(\mathbf{q}) \\
            {\mathbf{v}}^\times
        \end{bmatrix}}_{\mathbf{f}_4}\boldsymbol{\omega}_f,\quad\text{where}\ \mathbf{J}^{\pm}(\mathbf{q}) = \begin{bmatrix}
        q_4 \mathbf{1} \pm \mathbf{q}_{1:3}^\times \\
        -\mathbf{q}_{1:3}
    \end{bmatrix}.
\end{equation}

\subsection{The Inter-drone Distance-only Observation Model}

In our CLS, the leader determines its position and velocity through GNSS\@.
The follower is equipped with no global positioning sensors except a UWB sensor, which provides only an inter-drone distance measurement \(d = \norm{\mathbf{r}}\).
The attitudes of the leader and the follower are measured by their individual IMU and compasses.
This information is shared between the leader and follower either by extending the UWB system to transmit richer data alongside range measurements or via Wi-Fi, allowing the CLS to compute the relative pose $\mathbf{q}$.
The observation model is as follows:
\begin{equation}
    \label{eq:observationModel}
    \mathbf{y} = \mathbf{h}(\mathbf{x}) \triangleq \begin{bmatrix}
        \frac{1}{2}{d}^2 \\
        \mathbf{q}
    \end{bmatrix} \in
    \mathbb{R}^{5 \times 1}.
\end{equation}
where $\frac{1}{2}{d}^2$ is used to avoid numerical problems without changing the available information about inter-drone distance.

\subsection{The Follower Drone Localization Problem with Range-Only Measurements}

For a nonlinear system such as a quadrotor CLS, state estimators are typically not optimal, and estimation precision can be improved by optimal control.
Formulating the optimization problem requires a principled choice of both the precision metric and the system model.

We chose observability, quantified by the LOG and approximated through Lie Derivatives following Ref.~\cite{hausman2017observability}, as the precision metric because Lie derivatives enable it to reflect how the system configuration, velocity, acceleration, and even higher-order motion lead to informative, evolving observations.
Many alternative metrics of positioning quality, such as EKF covariance, the FIM, and the CRLB, rely on first-order linearizations of the observation model.
They are generally unsuitable for systems with indirectly observed state components, such as inertial navigation models augmented by biases~\cite{hausman2017observability,grebe2021observability} and our CLS, where the inter-drone velocity \(\mathbf{v}\) is absent from Eq.~\eqref{eq:observationModel}.

For example, consider the EKF posterior covariance matrix \(\hat{\mathbf{P}}_k\), used in Refs.~\cite{trawny2004optimized,go2024trajectory} and shown below:
\begin{gather}
    \min_{\mathbf{u}_1,\ldots,N} \gamma\left(\hat{\mathbf{P}}_N\right) \\
    \st\ \hat{\mathbf{P}}_k^{-1} = \mathbf{H}_k^\top \boldsymbol{\Sigma}\mathbf{H}_k + \check{\mathbf{P}}_{k-1},\quad \mathbf{x}_{k+1} = \mathbf{f}_d(\mathbf{x}_k,\mathbf{u}_k) \label{eq:EKFUpdate}
\end{gather}
where \(\gamma\) is a metric of choice, e.g.\, \(\tr, \det\), or minimum eigenvalue \(\lambda_{\min}\); \(\mathbf{H}_k\) is the observation Jacobian and \(\check{\mathbf{P}}_0\) is the initial prior covariance.
Applied to our CLS, \(\mathbf{H}_k\) is computed by differentiating Eq.~\eqref{eq:observationModel} with respect to \(\mathbf{x}_k\), giving:
\begin{equation}
    \mathbf{H}_k = D \mathbf{h}(\mathbf{x}_k) = \begin{bmatrix}
        \mathbf{r}_k^\top & \mathbf{0} & \mathbf{0} \\
        \mathbf{0}        & \mathbf{1} & \mathbf{0}
    \end{bmatrix} \in \mathbb{R}^{5\times 10},
\end{equation}
which causes \(\mathbf{H}_k^\top \boldsymbol{\Sigma}\mathbf{H}_k\), also known as the one-step FIM, to have a constantly zero bottom-right block.
When the Kalman covariance is updated according to Eq.~\eqref{eq:EKFUpdate}, \(\hat{\mathbf{P}}_k\) cannot change from \(\check{\mathbf{P}}_0\) in the velocity subspace, and no change in velocity estimation uncertainty is ever captured.
The broader issue, that first-order metrics cannot capture the coupling between motion and observability in indirectly observed states, is shared by other metrics shown in Table~\ref{tab:firstOrderMetrics}.

\begin{table}[htbp]
    \centering
    \caption{Failure modes of first-order metrics under the presence of indirectly observed states.}%
    \label{tab:firstOrderMetrics}
    \begin{tabularx}{\linewidth}{Xccc} \toprule
        First-Order Metrics & EKF covariance                                     & FIM                                                               & CRLB                      \\ \midrule
        Expression          & \(\hat{\mathbf{P}}_k\), see Eq.~\ref{eq:EKFUpdate} & \(\mathbf{I} = \mathbf{H}_k^\top\boldsymbol{\Sigma}\mathbf{H}_k\) & \(\det{\mathbf{I}}^{-1}\) \\ \midrule
        Failure mode        & Unchanged from prior in subspaces                  & Rank-deficient due to zero blocks                                 & Undefined                 \\ \bottomrule
    \end{tabularx}
\end{table}

We chose inter-drone dynamics as our system model to ensure the realizability of optimized motion strategies.
The unicycle model, widely used in CLS, is not well-suited to quadrotors, since its motion is coupled to heading, but quadrotor motion is governed by thrust through pitch and roll.
Another commonly used simple model is the first-order kinematic system.
Such models have proven effective for offline planning of quadrotor trajectories, where spline parameterization~\cite{salaris2017online,de2021online} or analytic guidance laws~\cite{li2022three} enforce smoothness and dynamical feasibility.
However, the present work considers a receding-horizon optimal control formulation, in which control inputs are optimized directly at each timestep.
Unaided by trajectory parameterization, kinematic models no longer satisfy implicit smoothness.
If observability-based objectives induce sharp variations in the velocity commands to excite the system for informative observations, then quadrotors, which cannot instantaneously change their velocity, may struggle to track them.

We demonstrate this phenomenon in a diagnostic simulation where our leader-follower CLS is instantiated with a first-order kinematic model in a receding-horizon optimal control setting.
The follower's \emph{velocity} commands at each timestep are optimized under the minimum eigenvalue of the ELOG~\cite{krener2009measures} and E\textsuperscript{2}LOG~\cite{hausman2017observability} over one predictive horizon, without trajectory parameterization, and bounded below 10 m/s.
The accelerations implied by the resulting velocity commands are plotted in Fig.~\ref{fig:unattainableAccelerations}.
Despite the bounds on velocity, accelerations peaked above \(50 \mathrm{m/s^2}\) and remained above \(20 \mathrm{m/s^2}\) --- the maximum achievable by our F450 experimental platforms --- for up to 70\% and 90\% of the time horizon under ELOG and E\textsuperscript{2}LOG, respectively.
This is compared to safe acceleration levels below \(5 \mathrm{m/s^2}\), produced by our optimal controller using quadrotor dynamics in subsequent simulations in Sec.~\ref{sec:experimental}, and also displayed in Fig.~\ref{fig:unattainableAccelerations}.

\begin{figure}[htbp]
    \centering
    \includegraphics{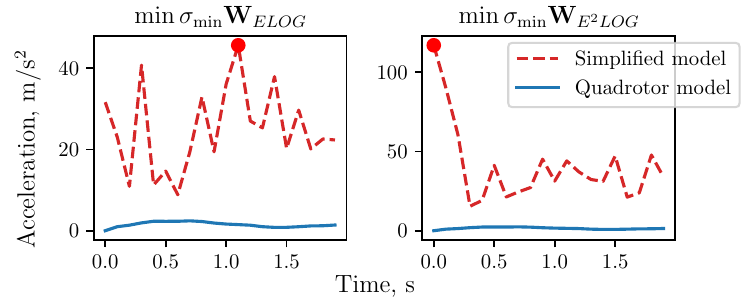}%
    \caption{Translation acceleration required by optimized commands for a first-order kinematic model}%
    \label{fig:unattainableAccelerations}
\end{figure}

These results indicate that, in an optimal control setting without trajectory parameterization, simplified models like first-order kinematic models may yield observability-seeking control commands whose implied accelerations are unattainable by quadrotors.
This motivates our choice of the more realistic quadrotor dynamics in Eq.~\eqref{eq:dynamicsModel} in which thrust and body-rate inputs are optimized directly and are natively compatible with quadrotor low-level controllers.

Subsequently, we start by proving that classical local weak observability is attainable by our range-only quadrotor CLS in Section~\ref{sec:observabilityAnalysis} to establish the viability of this system.
Then, we address the issue of approximating the LOG efficiently and reasonably through the STLOG in Section~\ref{sec:measurementOfObservability}and use it to develop a control principle in Section~\ref{sec:OAC}.

% Add recap to link to STLOG based problem formulation

\section{Observability Analysis with Range-Only Measurement}%
\label{sec:observabilityAnalysis}

In this section, we review the classical nonlinear observability analysis and perform it on our CLS to prove that it satisfies weak local observability.
We also introduce a local observability index to prepare for investigating whether weak local observability can be automatically satisfied by observability-aware optimal control.
We use the definition from Ref.~\cite{antsaklis2006linear}, that a system is observable if and only if the Observability Gramian (OG)
\begin{equation}
    \label{eq:observabilityGramian}
    \mathbf{W} = \int_0^T \boldsymbol{\Phi}_t^\top \mathbf{H}^\top \mathbf{H} \boldsymbol{\Phi}_t dt,
\end{equation}
is nonsingular.
For a nonlinear system, \(\mathbf{W}\) has no closed-form expression and varies with the \emph{local} point of evaluation, so it is referred to as the LOG~\cite{krener2009measures}.
In practice, Hermann and Krener's~\cite{hermann1977nonlinear} \emph{rank test} is used to detect the observability of a dynamical system.
For control-affine systems like Eq.~\eqref{eq:dynamicsModelControlAffine}, the r\textsuperscript{th}-order nonlinear Observability Matrix is:
\begin{equation}
    \label{eq:observabilityMatrix}
    \boldsymbol{\mathcal{O}}^{(r)}(\mathbf{x}) \triangleq
    \begin{bmatrix}
        \vdots                                                     \\
        DL^{({k})}_{\mathbf{f}_i, \ldots, j}\mathbf{h}(\mathbf{x}) \\
        \vdots
    \end{bmatrix}\quad i,\ldots,j = 0,\ldots,m,\quad k = 0,\ldots,r
\end{equation}
and the system is \emph{weakly locally observable} at~\(\mathbf{x}\) if $\rank \boldsymbol{\mathcal{O}}^{(r)}(\mathbf{x}) = \dim \mathcal{X}$ for some $r$, indicating the existence of control inputs that distinguish initial conditions locally.
We thereby analyze the observability of our system in Proposition~\ref{prop:observability}.

% Missing
\begin{prop}\label{prop:observability}
    The range-only quadrotor CLS given by Eqs.~\eqref{eq:dynamicsModel} and~\eqref{eq:observationModel} is weakly locally observable, since its 3\textsuperscript{rd}-order observability matrix may be expressed as:
    \begin{equation}\label{eq:observabilityMatrixCLS}
        \boldsymbol{\mathcal{O}}^{(3)}(\mathbf{x}) = \left[\begin{array}{ccc}
                {\mathbf{r}}^\top                              & \mathbf{0}_{1\times 4}       & \mathbf{0}_{1\times 3} \\
                \mathbf{0}_{4\times 3}                         & \mathbf{1}_{4\times 4}       & \mathbf{0}_{4\times 3} \\\midrule
                {\mathbf{v}}^\top                              & \mathbf{0}_{1\times 4}       & {\mathbf{r}}^\top      \\\midrule
                \mathbf{0}_{1\times 3}                         & \mathbf{0}_{1\times 4}       & 2{\mathbf{v}}^\top     \\\midrule
                \mathbf{1}_3^\top\mathbf{R}{(\mathbf{q})}^\top & {\mathbf{r}}^\top \mathbf{J} & \mathbf{0}_{1\times 3} \\\midrule
                -\mathbf{1}_3^\top                             & \mathbf{0}_{1\times 4}       & \mathbf{0}_{1\times 3} \\\midrule
                \mathbf{0}_{1\times 3}                         & \mathbf{0}_{1\times 4}       & -2\mathbf{1}_3^\top
            \end{array}\right],
    \end{equation}

    This matrix attains a rank of 10.
    Therefore, the system satisfies the observability rank condition.
    \begin{proof}
        The partitions of \(\boldsymbol{\mathcal{O}}^{(3)}(\mathbf{x})\) in Eq.~\eqref{eq:observabilityMatrixCLS} are the nonzero rows of \(DL\mathbf{h}\), \(DL_{\mathbf{f}_0}\mathbf{h}\), \(DL_{\mathbf{f}_0}^2 \mathbf{h}\), \(DL_{\mathbf{f}_3 \mathbf{f}_0}\mathbf{h}\), \(DL_{\mathbf{f}_1 \mathbf{f}_0}\mathbf{h}\), and \(DL_{\mathbf{f}_3}L_{\mathbf{f}_0}^2 \mathbf{h}\), computed by the MATLAB symbolic toolbox\footnote{The relevant code can be found in \url{https://github.com/FSC-Lab/observability_aware_controller/symbolic_algebra}}.
        By row reduction, it is found to have a rank of 10.

    \end{proof}
\end{prop}

We can identify that the rank of \(\boldsymbol{\mathcal{O}}^{(3)}(\mathbf{x})\) falls below 10 under certain conditions in state \(\mathbf{x}\), such as when \(\mathbf{r}\) becomes a scalar multiple of \(\mathbf{v}\).
Although we do not investigate these conditions in detail, we pose the question of whether observability-aware control can direct the state of the system to satisfy weak local observability.
To this end, we define \emph{local observability index} as the minimum order of Lie derivatives necessary for \(\boldsymbol{\mathcal{O}}^{(r)}(\mathbf{x}, \mathbf{u})\) to attain full rank:
\begin{equation}\label{eq:rstardef}
    r_*(\mathbf{x}, \mathbf{u}) = \inf \left\{ r \ \vert \ \rank \boldsymbol{\mathcal{O}}^{(r)}(\mathbf{x}, \mathbf{u}) = \dim \mathcal{X} \right\},
\end{equation}
where:
\begin{equation}\label{eq:observabilityMatrixControl}
    \boldsymbol{\mathcal{O}}^{(r)}(\mathbf{x}, \mathbf{u}) = \begin{bmatrix}
        \vdots \\  D{L_{\mathbf{f}}^{(k)} \mathbf{h}(\mathbf{x},\mathbf{u})} \\ \vdots
    \end{bmatrix},\quad k = 0,\ldots,r,
\end{equation}
is another observability matrix, parameterized by inputs but free of the control-affine requirement, as seen in~\cite{hausman2017observability,yu2015observability}.

Thus, we proved that the quadrotor CLS satisfies weak local observability, and defined the local observability index to help us investigate whether this property is guaranteed under observability-aware control in Section~\ref{sub:STLOGranktesting}.

\section{The Measure of Observability via STLOG}%
\label{sec:measurementOfObservability}

In this section, we consider the problem of estimating a system's initial condition from noisy observations along its trajectory, according to the definition of observability, and then derive the STLOG based on our findings.

\begin{figure}[htpb]
    \centering
    \begin{tikzpicture}[]

    % Irregular signal function definition
    \pgfmathdeclarefunction{irregular}{1}{%
        \pgfmathparse{0.3*sin(1.1*#1 r) + 0.2*sin(2.3*#1 r + 15) + 0.1*cos(3.7*#1 r)}%
    }

    % Continuous signal with uncertainty band
    \begin{scope}
        \clip (0,-2) rectangle (7,2);
        \draw[fill=gray!20, draw=none]
        plot[smooth, domain=0:7, samples=200] (\x, {irregular(\x) + 0.5})
        -- plot[smooth, domain=7:0, samples=200] (\x, {irregular(\x) - 0.5})
        -- cycle;
        \draw[thick, blue]
        plot[smooth, domain=0:9, samples=200] (\x, {irregular(\x)});
    \end{scope}
    % Axes
    \draw[->] (0,-2) -- (0,2) node[above] {$y$};
    \draw[->] (0,-2) --  (0.45\linewidth,-2) node [right] {$t$};
    \coordinate (cutoff) at (7, {irregular(7)});
    \draw (cutoff) -- (7, -2) node [above right] {$T$};
    % GP annotation

    % Discrete observations
    \foreach \x in {1.5,3,4.5} {
            \pgfmathsetmacro{\y}{irregular(\x)}
            \draw[blue, line width=0.75pt] (\x, \y + 0.7) -- (\x, \y - 0.7);
            \draw[blue, line width=0.75pt] (\x-0.1, \y + 0.7) -- (\x+.1, \y + 0.7);
            \draw[blue, line width=0.75pt] (\x-0.1, \y - 0.7) -- (\x+.1, \y - 0.7);
            \draw[cyan, line width=2.5pt] (\x, \y + 0.5) -- (\x, \y - 0.5);
        }

    % Classical arrows
    \coordinate (labelAnchor) at (1.5,{irregular(1.5)});
    \coordinate (labelAnchor1) at (3.5,{irregular(3.5)});
    \coordinate (labelAnchor2) at (3,{irregular(3)});
    \coordinate (labelAnchor3) at (4.5,{irregular(4.5)});
    \coordinate (labelAnchor6) at (5,{irregular(5)});
    \node[above,align=center, text width=0.4\linewidth] (oursEnd) at ($(labelAnchor1) + (0, 0.8)$) {(Ours) Continuous-time observation process \(\mathbf{y}(t)\) with GP noise \(\boldsymbol{\eta}(t)\) enables inference without relying on discrete observations};
    \node[below, align=center] (classicalEnd) at ($(labelAnchor) + (1, -1.4)$) {(Classical) Discrete observations \\dictate uncertainty propagation};
    \draw [->, PineGreen, thick] (oursEnd) -- (labelAnchor1);

    \foreach \it in {labelAnchor, labelAnchor2, labelAnchor3} {
            \draw[->, red, thick] (\it) -- (classicalEnd);
        }

    % \draw[->, red, thick] (4, \yclass) -- ++(1.0,-1.0);

    % Inset posterior ellipse
    \begin{scope}[shift={($(labelAnchor6) + (1.0, -1.4)$)}]
        \coordinate (ellipseCenter) at (0.0,0.);
        \node[draw, fill=white, opacity=1, label={[align=center]below:(Ours) posterior\\inferred at any \(t\in[0, T]\)}, minimum width=1cm, minimum height=0.8cm] (covLabel) at (ellipseCenter) {};
        \fill[green, opacity=0.4] (ellipseCenter) ellipse (0.4 and 0.3);
        \node at (ellipseCenter) {$\hat{\mathbf{x}}_0$};
    \end{scope}
    \draw[thick, ->] (covLabel) -- (labelAnchor6);

\end{tikzpicture}
    \caption{Visualizing initial state estimation by (our) observation process and (classical) discrete observations.}%
    \label{fig:oacViz}
\end{figure}

We model an uncertain initial condition as $\mathbf{x}_0\sim \mathcal{N}\left(\mathbf{x}_0^*, \check{\mathbf{P}}_0\right)$.
Across the horizon \([0, T]\), the system receives observations $\mathbf{y}(t)$.
Since we are deriving the LOG, a continuous-time quantity, we depart from the classical state estimation approach of analyzing a few discrete observations as illustrated in Fig.~\ref{fig:oacViz}, and instead model the observation noise $\boldsymbol{\eta}(t)$ as a GP\@.
We represent the estimate for $\mathbf{x}_0$ given $\mathbf{y}(t)$ by the posterior distribution $p(\mathbf{x}_0 \mid \mathbf{y}(t))$, which is factored by Bayes' rule:
\begin{equation}\label{eq:bayesrule0}
    p(\mathbf{x}_0 \mid \mathbf{y}(t) ) = \frac{p(\mathbf{y}(t) \mid \mathbf{x}_0)}{p(\mathbf{y}(t))} p(\mathbf{x}_0),
\end{equation}

The key challenge in this problem is that the likelihood ratio ${p(\mathbf{y}(t) | \mathbf{x}_0)}/{p(\mathbf{y}(t))}$ is between \emph{infinite-dimensional distributions}, which is less tractable for analysis and computation than the elementary, finite-dimensional case.
We outline these difficulties heuristically and defer a full rigorous treatment using Gaussian measure theory to Appendix~\ref{app:CameronMartin}.

We consider the covariance kernel $\mathbf{K}(t, t^\prime) = E\left[\boldsymbol{\eta}(t) \boldsymbol{\eta}{(t^\prime)}^\top \right]$ of GP $\boldsymbol{\eta}(t)$, and we consider the map from an initial condition to a function in time of motion-coupled observations to be the \emph{system response} $\mathcal{A}: \mathcal{X} \rightarrow C^\infty([0,T], \mathbb{R}^p)$, defined as:
\begin{equation}\label{eq:systemResponse}
    \left( \mathcal{A}(\mathbf{x}_0) \right)(t) = \mathcal{A}_t(\mathbf{x}_0) = \mathbf{h} \left(\phi_t(\mathbf{x}_0)\right) = \mathbf{h} \left(\mathbf{x}(t)\right),
\end{equation}

Subject to the system model in Eq.~\eqref{eq:dynamicalSystem} and~\eqref{eq:observation}, the conditional observation $\mathbf{y}(t) | \mathbf{x}_0$ is also a GP, with mean $\mathcal{A}_t(\mathbf{x}_0)$ and covariance $\mathbf{K}(t, t^\prime)$.
This enables us to consider a subproblem of computing the likelihood ratio between two initial conditions, ${p(\mathbf{y}(t) | \mathbf{x}_0)}/{p(\mathbf{y}(t) | \mathbf{x}_1)}$.
Recall from Mercer's theorem in Ref.~\cite{seeger2004gaussian} that the covariance kernel $\mathbf{K}(t, t^\prime)$, under suitable assumptions, defines an integral operator $\mathcal{K}$ in the Hilbert space $L^2([0,T], \mathbb{R}^p)$ of vector-valued square-integrable functions on $[0,T]$.
The operator $\mathcal{K}$ has a complete set of orthonormal eigenfunctions $\left\{ \mathbf{v}_n(t) \right\}_{n=0}^\infty$ with respect to the inner product \(\langle\cdot,\cdot\rangle\) in \({L^2([0,T], \mathbb{R}^p)}\), with corresponding eigenvalues $\left\{ \lambda_n \right\}_{n=0}^\infty$ such that $(\mathcal{K} \mathbf{v}_n)(t) = \lambda_n \mathbf{v}_n(t)$.

Let $y_{n}\lvert_i =\langle \mathbf{v}_n(t), \mathbf{y}(t) \rangle | \mathbf{x}_i$ be the component of $\mathbf{y}(t) | \mathbf{x}_i$ in the eigenfunction basis defined by $\mathcal{K}$.
Each $y_{n}\lvert_i$ is then a one-dimensional Gaussian variable, with mean $\langle \mathbf{v}_n(t), \mathcal{A}_t(\mathbf{x}_i) \rangle$ and variance $\lambda_n$.
Moreover, $y_{n}\lvert_i$ and $y_{m}\lvert_i$ are independent of $n \neq m$.
An elementary attempt to evaluate ${p(\mathbf{y}(t) | \mathbf{x}_0)}/{p(\mathbf{y}(t) | \mathbf{x}_1)}$ is to treat it as the infinite product of the component-wise likelihood ratios:
\begin{align}\label{eqn:heuristic_likelihood_ratio}
    \frac{p(\mathbf{y}(t) | \mathbf{x}_0)}{p(\mathbf{y}(t) | \mathbf{x}_1)} \overset{heuristic}{=} \frac{\Pi_{n=0}^\infty p(y_{n}\lvert_0)}{\Pi_{n=0}^\infty p(y_{n}\lvert_1)} \overset{heuristic}{=} \Pi_{n=0}^\infty\frac{ p(y_{n}\lvert_0)}{ p(y_{n}\lvert_1)},
\end{align}
where we have annotated the heuristic nature of the relation between each expression.
The heuristic expressions in Eq.~\eqref{eqn:heuristic_likelihood_ratio} elicit questions regarding:
\begin{enumerate}
    \item Whether the `joint probability distribution' $ \Pi_{n=0}^\infty p(y_{n}\lvert_i)$ exists meaningfully; and if it does, whether its samples belong to a sufficiently regular function space, such as \({L^2([0,T], \mathbb{R}^p)}\).
    \item Whether the infinite product of likelihood ratios $\Pi_{n=0}^\infty ({ p(y_{n}\lvert_0)}/{ p(y_{n}\lvert_1)})$ converges.
\end{enumerate}

We address these questions rigorously in Appendix~\ref{app:CameronMartin}, but state the main result for proceeding with the initial condition estimation problem in Proposition~\ref{prop:likelihood_maintextver}.

% NOTE: Turn this into a proposition, referring to rigorous treatment in Appendix.
\begin{prop}\label{prop:likelihood_maintextver}
    Under suitable assumptions on the system response $\mathcal{A}$ in Eq.~\eqref{eq:systemResponse} and the covariance operator $\mathcal{K}$, the likelihood ratio ${p(\mathbf{y}(t) | \mathbf{x}_0)}/{p(\mathbf{y}(t))}$ can be expressed in closed-form as:
    \begin{equation}\label{eqn:likelihoodratio1}
        \frac{p(\mathbf{y}(t) | \mathbf{x}_0)}{p(\mathbf{y}(t))} = C(\mathbf{y}(t)) \exp\left( \langle \mathbf{y}(t), \mathcal{K}^{-1}\mathcal{A}_t(\mathbf{x}_0) \rangle - \frac{1}{2} \langle \mathcal{A}_t(\mathbf{x}_0), \mathcal{K}^{-1}\mathcal{A}_t(\mathbf{x}_0) \rangle \right).
    \end{equation}
    where \(\langle\cdot,\cdot\rangle\) denotes the inner product on \({L^2([0,T], \mathbb{R}^p)}\).

    \begin{proof}
        The conditions for Eq.~\eqref{eqn:likelihoodratio1} to be derived is given in Proposition~\ref{prop:main_likelihood_ratio} and Corollary~\ref{cor:main_likelihood_ratio}
    \end{proof}
\end{prop}

% NOTE: Keep this definition in the main text
Next, we introduce a practical noise model, which, when paired with the assumption that the system response $\mathcal{A}: \mathcal{X} \rightarrow C^\infty([0,T], \mathbb{R}^p)$ in Eq.~\eqref{eq:systemResponse} is continuous with respect to the Fr\'{e}chet topology on $C^\infty([0,T], \mathbb{R}^p)$\footnote{
A characterization of the Fr\'{e}chet topology on $C^\infty([0,T], \mathbb{R}^p)$ is as the metric topology of \(d(\mathbf{f}(t), \mathbf{g}(t)) = \sum_{k=0}^\infty 2^{-k} \frac{\rho_k(\mathbf{f} - \mathbf{g} )}{1 + \rho_k(\mathbf{f} - \mathbf{g}) }\)~\cite{reed1980methods} where \(\rho_k(\mathbf{f} - \mathbf{g}) = {\sup}_{t \in [0,T] }\norm{\frac{d^k}{dt^k} (\mathbf{f}(t) - \mathbf{g}(t))}.
\)
In particular, all linear systems will have a continuous system response, owing to the existence of an explicit form $\mathcal{A}_t(\mathbf{x}_1) - \mathcal{A}_t(\mathbf{x}_0) = \mathbf{M}(t)\left(\mathbf{x}_1 - \mathbf{x}_0\right)$, for some analytically known smooth matrix-valued function $\mathbf{M}(t)$ on $[0,T]$.
}
, satisfies the prerequsites of Proposition~\ref{prop:likelihood_maintextver}.
We achieve this by specifying the covariance operator $\mathcal{K}$ spectrally, using the orthonormal Legendre polynomials ${\left\{e_n(t) \right\}}_{n=0}^\infty$ on $[0, T]$.
The detailed rationale is explained in Corollary~\ref{cor:main_likelihood_ratio}.

\begin{defn}[Band-limited white noise model with pink tail]\label{def:eta1}
    \begin{equation}
        \boldsymbol{\eta}(t) = \frac{1}{\sqrt{N_c}} \sum_{n=0}^{N_c-1} \boldsymbol{\chi}_n e_n(t) + \sum_{n = N_c}^{\infty} \sqrt{a_n n^{-\alpha}} \boldsymbol{\chi}_n e_n(t),
    \end{equation}
    where
    \begin{enumerate}
        \item $N_c > 1$: Bandwidth parameter, signifying the number of `active' Legendre components
        \item $\alpha > 1$: High-frequency decay exponent
        \item \({\{\boldsymbol{\chi}_n\}}_0^\infty\): A sequence of independent, vector-valued spatial noise on each Legendre component, each following the Gaussian distribution \(\mathcal{N}(0,\boldsymbol{\Sigma})\), where \(\boldsymbol{\Sigma}\) is a positive-definite covariance matrix.
        \item \({\{a_n\}}_{N_c}^\infty\): A sequence of tail coefficients.
              Positive constants $a_{-}, a_{+}$ must exist such that \(a_n \in [a_{-}, a_{+}]\) for all $n$.
    \end{enumerate}
\end{defn}

Definition~\ref{def:eta1} approximately defines a GP with $N_c$ `active' components, each following the prescribed Gaussian distribution \(\mathcal{N}(0,\boldsymbol{\Sigma})\) and with equal energy $1/N_c$, plus a power-law decay at frequencies above $1/\delta t = N_c/T$.
More precisely, let ${\left\{\sigma_i \right\}}_{i=1}^{p}$ and ${\left\{\mathbf{v}_i \right\}}_{i=1}^{p}$ be the sets of eigenvalues and orthonormal eigenvectors for \(\boldsymbol{\Sigma}\), respectively.
Then ${\left\{\mathbf{v}_i e_n(t) \right\}}$ is a complete set of orthonormal eigenfunctions for the covariance operator $\mathcal{K}$, with eigenvalues:
\begin{equation}\label{eqn:eta2}
    \mathcal{K}\left( \mathbf{v}_i e_n(t)\right) = \begin{cases}
        \frac{1}{N_c} \sigma_i \mathbf{v}_i e_n(t)  & n < N_c,    \\
        a_n n^{-\alpha} \sigma_i\mathbf{v}_i e_n(t) & n \geq N_c.
    \end{cases}
\end{equation}

We follow Corollary~\ref{cor:main_likelihood_ratio} to check that $\mathcal{K}$ in Eq.~\eqref{eqn:eta2} satisfies the relevant hypotheses in Corollary~\ref{cor:main_likelihood_ratio} and apply Proposition~\ref{prop:likelihood_maintextver} to express the likelihood ratio as Eq.~\eqref{eqn:likelihoodratio1}. Then, we further simplify it in Proposition~\ref{prop:likelihood_simplify}.
\begin{prop}\label{prop:likelihood_simplify}
    Using the noise model in Definition~\ref{def:eta1}, for sufficiently large $N_c$, the likelihood ratio in Eq.~\eqref{eqn:likelihoodratio1} at a given $\mathbf{x}_0$ and a given $\mathbf{y}(t)$ can be written as
    \begin{equation}\label{eq:bayeslikelihood}
        \frac{p(\mathbf{y}(t) | \mathbf{x}_0)}{p(\mathbf{y}(t))} = C(\mathbf{y}(t)) \exp\left(\frac{N_c}{T}\left( \int_0^T \mathbf{y}^\top(t) \boldsymbol{\Sigma}^{-1}\mathcal{A}_t(\mathbf{x}_0) - \frac{1}{2} \mathcal{A}_t^T(\mathbf{x}_0) \boldsymbol{\Sigma}^{-1}\mathcal{A}_t(\mathbf{x}_0) \ dt \right) + {O}(N_c^{-k})\right),
    \end{equation}
    for all $k > 0$.
\end{prop}
\begin{proof}
    Since \(\mathcal{A}(\mathbf{x}_0)(t)\) is smooth in $t$, its $n$-th Fourier--Legendre coefficient is ${O}(n^{-k})$ for all $k>0$.
    Using the explicit expression for $\mathcal{K}$ in Eq.~\eqref{eqn:eta2}, we compute that, for sufficiently large $N_c$,
    \begin{align}\label{eqn:likelihoodcal_decay1}
        \left(\mathcal{K}^{-1}\mathcal{A}(\mathbf{x}_0)\right)(t) = N_c \boldsymbol{\Sigma}^{-1}\mathcal{A}_t(\mathbf{x}_0) + {O}(N_c^{-k}),
    \end{align}
    for all $k > 0$. Substituting Eq.~\eqref{eqn:likelihoodcal_decay1} into Eq.~\eqref{eqn:likelihoodratio1} gives the desired result.
\end{proof}
In the following analysis, we shall neglect the ${O}(N_c^{-k})$ terms in Eq.~\eqref{eq:bayeslikelihood}.
% \footnote{
% The result (\ref{eq:bayeslikelihood}) coincides with the naive lattice calculation where we assume ${E}\left[  \boldsymbol{\eta}(t_i)  \boldsymbol{\eta}^\top(t_j)\right] = \boldsymbol{\Sigma} \delta_{ij}$ on the lattice $t_i = i \delta t, i = 1,\ldots, N_c$.
% The reason why we avoid the lattice model or any equivalent finite-dimensional parameterization of $\boldsymbol{\eta}(t)$ as a linear image of a finite-dimensional Gaussian variable is that ${p(\mathbf{y}(t) | \mathbf{x}_0)}/{p(\mathbf{y}(t) | \mathbf{x}_1)}$ becomes ill defined whenever $\mathcal{A}_t(\mathbf{x}_0) - \mathcal{A}_t(\mathbf{x}_1)$ does not lie in this image.
% See Appendix~\ref{app:CameronMartin} and the references therein for more details.
% }
We now describe the structure of the posterior logarithmic likelihood \(\log p(\mathbf{x}_0\mid \mathbf{y}(t))\) as a function of $\mathbf{x}_0$, up to constants \(C_1, C_2\) independent of $\mathbf{x}_0$.
We express this equivalent class using Eqs.~\eqref{eq:bayesrule0}and~\eqref{eq:bayeslikelihood}:
\begin{align}\label{eq:logposterior1}
     & \log p(\mathbf{x}_0 |\mathbf{y}(t) ) + C = \log \frac{p(\mathbf{y}(t) | \mathbf{x}_0)}{p(\mathbf{y}(t))}  + \log p(\mathbf{x}_0) + C = \log \left(\frac{p(\mathbf{y}(t) | \mathbf{x}_0)}{p(\mathbf{y}(t))} {\left(\frac{p(\mathbf{y}(t) | \mathbf{x}^*_0)}{p(\mathbf{y}(t))} \right)}^{-1}\right) + \log p(\mathbf{x}_0) + C' \nonumber                        \\
     & = \frac{N_c}{T} \int_0^T \left( {(\mathbf{y}(t) - \mathcal{A}_t(\mathbf{x}_0^*))}^T \boldsymbol{\Sigma}^{-1}(\mathcal{A}_t(\mathbf{x}_0) - \mathcal{A}_t(\mathbf{x}_0^*)) - {(\mathcal{A}_t(\mathbf{x}_0) - \mathcal{A}_t(\mathbf{x}_0^*))}^T\frac{1}{2}\boldsymbol{\Sigma}^{-1}(\mathcal{A}_t(\mathbf{x}) - \mathcal{A}_t(\mathbf{x}_0^*))\right)dt \nonumber \\
     & \quad -\frac{1}{2}{(\mathbf{x}-\mathbf{x}_0^*)}^T \check{\mathbf{P}}_0^{-1}(\mathbf{x}-\mathbf{x}_0^*),
\end{align}
where we introduced the mean of the initial condition \(\mathbf{x}_0^*\) as a reference point and expressed the likelihood ratio relative to it to eliminate the intractable marginal likelihood \(p(\mathbf{y}(t))\).
Since \(\mathbf{x}_0^*\) is fixed, its associated likelihood ratio \(\frac{p(\mathbf{y}(t) | \mathbf{x}^*_0)}{p(\mathbf{y}(t))} \) contributes only an additive constant, which is absorbed into the change from \(C\) to \(C'\).
This preserves the equivalence class of the posterior up to terms independent of \(\mathbf{x}_0\), while allowing the remaining terms to be expressed in a tractable quadratic form.

% NOTE: Simply describe the validity of linearization under scaling
To simplify Eq.~\eqref{eq:logposterior1}, we linearize around the prior expectation $\mathbf{x}_0^*$.
Such a linearization is valid in the following scaling regime: for some scaling parameter ${\epsilon} \ll 1$,
\begin{equation}
    \delta \mathbf{x} = \mathbf{x_0}^* - \mathbf{x}_0 \thicksim {\epsilon} ,\quad \boldsymbol{\Sigma} \thicksim {\epsilon}^2, \quad\check{\mathbf{P}}_0 \thicksim {\epsilon}^2, \label{eqn:deltaxdef}
\end{equation}
where $\thicksim$ denotes a scaling relationship in orders of ${\epsilon}$.
Then the pathwise observation residual $\mathbf{z}(t) = \mathbf{y}(t) - \mathcal{A}_t(\mathbf{x}_0^*)$ also scales as ${\epsilon}$.
The linearized expression of the log-likelihood in Eq.~\eqref{eq:logposterior1}, expressed in terms of the displacement $\delta \mathbf{x}$ in Eq.~\eqref{eqn:deltaxdef}, is then expressed in a quadratic form:
\begin{equation}\label{eq:logposterior_linear}
    \log p(\delta \mathbf{x} |\mathbf{y}(t)) + C = \frac{N_c}{T}\mathbf{b}^\top\delta\mathbf{x} + \frac{1}{2}\delta \mathbf{x}^\top\left( \check{\mathbf{P}}_0^{-1} + \frac{N_c}{T}\mathbf{W}_{\boldsymbol{\Sigma}^{-1}} \right)\delta\mathbf{x},
\end{equation}
to leading order of ${\epsilon}$, where
\begin{equation}\label{eq:gramiansigmadef}
    \mathbf{W}_{\boldsymbol{\Sigma}^{-1}} = \int_0^T  D\mathcal{A}_t^\top\boldsymbol{\Sigma}^{-1}D\mathcal{A}_t dt, \quad\mathbf{b} = \int_0^T D\mathcal{A}_t^\top\boldsymbol{\Sigma}^{-1}\mathbf{z}(t) dt,
\end{equation}
and \(\mathbf{W}_{\boldsymbol{\Sigma}^{-1}}\) corresponds to the LOG with respect to the metric \(\boldsymbol{\Sigma}^{-1}\) in the observation space $\mathcal{Y}$.

\begin{rem}
    In the special case $\boldsymbol{\Sigma} = \mathbf{1}$ and $\mathcal{A}_t = \mathbf{h},\left(\phi_t(\mathbf{x}_0)\right)$ is smooth, substituting the classical system response map $D\mathcal{A}_t = \mathbf{H} \boldsymbol{\Phi}_t$ into Eq.~\eqref{eq:gramiansigmadef} gives the traditional definition of continuous time OG in Eq.~\eqref{eq:observabilityGramian}.
\end{rem}

Under this linear approximation, $\delta \mathbf{x}$ becomes a normally distributed random variable, with expectation and covariance
\begin{equation}
    {E}[ \delta\mathbf{x} | \mathbf{y}(t)] = \mathbb{V}[\delta \mathbf{x} | \mathbf{y}(t)]\left(\frac{N_c}{T} \mathbf{b}\right), \quad \mathbb{V}[ \delta\mathbf{x} | \mathbf{y}(t)] = {\left(\frac{N_c}{T}{\mathbf{W}_{\boldsymbol{\Sigma}^{-1}}} + \check{\mathbf{P}}_0^{-1} \right)}^{-1}.
\end{equation}

We conclude that in a small-noise regime where the scaling in Eq.~\eqref{eqn:deltaxdef} holds, the posterior of the initial condition estimate under a given observation path $\mathbf{y}(t)$ at times $[0,T]$ is approximately Gaussian $\mathbf{x}_0 \sim \mathcal{N}(\hat{\mathbf{x}}_0, \hat{\mathbf{P}}_0)$, where the posterior mean and covariance are given by
\begin{equation}
    \label{eq:ICestimation}
    \hat{\mathbf{x}}_0 - \mathbf{x}_0^* =  \frac{1}{\delta t}\hat{\mathbf{P}}_0 \mathbf{b},\quad \hat{\mathbf{P}}_0^{-1} = \frac{1}{\delta t}{\mathbf{W}_{\boldsymbol{\Sigma}^{-1}}} + \check{\mathbf{P}}_0^{-1},
\end{equation}
where $\mathbf{W}_{\boldsymbol{\Sigma}^{-1}}$ and $\mathbf{b}$ given in Eq.~\eqref{eq:gramiansigmadef}, and $\delta t = T/N_c$ is the correlation timescale of the noise term $\boldsymbol{\eta}(t)$.
Moreover, the relationship between prior and posterior covariances $\check{\mathbf{P}}_0, \hat{\mathbf{P}}_0$ in Eq.~\eqref{eq:ICestimation} is still valid if the prior distribution of $\mathbf{x}_0$ has a different prior mean $\check{\mathbf{x}}_0$, provided that the scaling $\check{\mathbf{x}}_0 - \mathbf{x}^*_0 \sim \epsilon$ holds, because the linearization argument for the log-posterior in Eq.~\ref{eq:logposterior1} remains valid.
In this regime, the minimum precision (inverse covariance) of the posterior distribution satisfies the following lower bound:
\begin{equation}\label{eq:icbound}
    {\lambda_{\max} (\hat{\mathbf{P}}_0)}^{-1} \geq \frac{\lambda_{\min}(\mathbf{W}_{\boldsymbol{\Sigma}^{-1}})}{\delta t} + {\lambda_{\max} (\check{\mathbf{P}}_0)}^{-1},
\end{equation}
where $\lambda_{\min/ \max}$ denotes the minimum/maximum eigenvalue of the relevant matrix.

Equation~\eqref{eq:icbound} enables us to evaluate arbitrarily long prospective trajectories based on their contribution to improving initial condition estimation precision.
To use this ability in a control context, we derive the STLOG as an approximation of the LOG based on \(\mathcal{A}_t\) from Eq.~\eqref{eq:systemResponse}.
This STLOG approximation is based on the following conditions
\begin{enumerate}
    \item The observation horizon $T$ is sufficiently short, but not shorter than the correlation timescale $\delta t$ of the noise $\boldsymbol{\eta}(t)$ in Definition~\ref{def:eta1}, so that $\delta t \ll T \ll 1$;\footnote{
              The STLOG approximation to the LOG is independent of the correlation timescale $\delta t$ \emph{per se}, but the ordering $\delta t \ll T$ is required for the LOG to emerge from Bayesian state estimation (c.f. Proposition~\ref{prop:likelihood_simplify}).
          }
    \item $\mathbf{u}(t)$ is constant for $t \in \halfopen{0}{T}$, so that our system dynamics in Eq.~\eqref{eq:dynamicalSystem} is autonomous within this horizon.
\end{enumerate}

The STLOG is fundamentally based on the expansion of the system response, defined as \(\mathcal{A}_t(\mathbf{x}_0) = \mathbf{h}(\mathbf{x}(t))\) in Eq.~\eqref{eq:systemResponse} as a Taylor series in $t$ about $t=0$.
Differentiating Eq.~\eqref{eq:systemResponse} with respect to $t$ under the conditions that $\mathbf{u}(t)$ is constant and $\mathbf{f}$ is time-independent gives the derivatives:
\begin{equation}\label{eq:systemResponseDerivatives}
    \frac{d}{dt}\mathcal{A}_t(\mathbf{x}_0) = \left(L_\mathbf{f} \mathbf{h} \right) (\mathbf{x}(t)), \quad
    \frac{d^j}{dt^j} \mathcal{A}_t(\mathbf{x}_0) = \left(L_\mathbf{f}^j \mathbf{h} \right) (\mathbf{x}(t)),
\end{equation}
where $L_\mathbf{f} = \mathbf{f} \cdot \nabla$ is the Lie derivative on functions\footnote{This is consistent with the definition of Lie derivatives in differential geometry if we consider $\mathbf{h}$ as a $p$-tuple of scalar functions in the manifold $\mathcal{X}$, whence no contractions involving $D\mathbf{f}$ are necessary.
    Note, however, that $D\mathbf{h}$ is considered as a $p$-tuple of one-forms on $\mathcal{X}$, and correspondingly $L_\mathbf{f} D\mathbf{h} = \mathbf{f} \cdot \nabla (D\mathbf{h}) + (D\mathbf{h}) (D\mathbf{f})$.
    The usual identity $L_\mathbf{f} D\mathbf{h} = D L_\mathbf{f} \mathbf{h}$ holds under this interpretation.}.
We express \(\mathcal{A}_t(\mathbf{x}_0)\) in terms of the derivatives in Eq.~\eqref{eq:systemResponseDerivatives} taken at $t=0$ to get:
\begin{equation}\label{eq:Aseries}
    \mathcal{A}_t(\mathbf{x}_0) = \sum_{j=0}^{\infty} \frac{t^j}{j!} \left(L_\mathbf{f}^j \mathbf{h} \right) (\mathbf{x}_0),
\end{equation}
which can be differentiated with respect to the state to give:
\begin{equation}\label{eq:DAseries}
    D\mathcal{A}_t = \sum_{j=0}^{\infty} \frac{t^j}{j!} D\left(L_\mathbf{f}^j \mathbf{h} \right).
\end{equation}

Substituting the series in Eq.~\eqref{eq:DAseries} into \(\mathbf{W}_{\boldsymbol{\Sigma}^{-1}}\) in~\eqref{eq:gramiansigmadef} and integrating yields the STLOG in Definition~\ref{def:stlog}.

\begin{defn}%
    \label{def:stlog}
    The order-$r$ \emph{short-term local observability Gramian (STLOG)} is the matrix
    % NOTE(CL-Chong) remove dependency on tangent space notation
    %quadratic form $\mathbf{W}^{(r)}_{\mathbf{R}^{-1}}: T_{\mathbf{x}^*_0} \mathcal{X} \times T_{\mathbf{x}^*_0} \mathcal{X} \to \mathbb{R}$
    \begin{equation}
        \mathbf{W}^{(r)}_{\boldsymbol{\Sigma}^{-1}} = \sum_{i,j = 0}^{r} \frac{{T}^{i+j+1}}{(i+j+1)i!j!} D{\left(L_\mathbf{f}^i \mathbf{h} \right)}^\top \boldsymbol{\Sigma}^{-1} D\left(L_\mathbf{f}^j \mathbf{h} \right). \label{eq:stlogExpr}
    \end{equation}
    which would be evaluated at \((\mathbf{x}^*, \mathbf{u}; T)\), respectively, the initial condition, control input, and observation time.
    As such, the STLOG is a novel point-wise representation of observability, which enables \emph{local} assessment of how the evolution of the system under the current control input contributes to the observability of the current operating state, and aligns with our goal of leveraging observability within predictive control frameworks designed in Section~\ref{sec:OAC}.

\end{defn}

\section{Observability-aware control}%
\label{sec:OAC}

We begin by deriving our observability-aware controller, the OPC, in Section~\ref{sub:OPC}.
We then investigate whether the resulting control strategy satisfies local weak observability in Section~\ref{sub:STLOGranktesting} to address the question raised in Section~\ref{sec:observabilityAnalysis}.

\subsection{OPC}\label{sub:OPC}

The goal of observability-aware control is to design a control strategy that refines the estimates of \(N\) states $\mathbf{x}(t)$ in any \(T_1, \ldots, T_k, \ldots,T_N\), separated by \(\Delta T\) in a discretization scheme.
We first propose a law that relates the prior and posterior covariance of state estimation at every \(T_k\) in a trajectory, by extending Eq.~\eqref{eq:ICestimation} in Proposition~\ref{prop:STLOGupdate}.
\begin{assum}\label{assum:piecewiseConstantInput}
    We assume the inputs $\mathbf{u}(t)$ are \emph{piecewise-constant} within each stage $\halfopen{T_k}{T_k + \Delta T}$, and moreover that the timescales follow the scaling $\delta t \ll \Delta T \ll T$, so we can employ STLOG in Eq.~\eqref{eq:stlogExpr} to approximate the LOG\@.
\end{assum}
\begin{assum}\label{assum:linearization}
    We assume that, for some scale parameter $\epsilon \ll 1$, both the state estimation covariances and $\boldsymbol{\Sigma}$ scale as $\epsilon^2$, as in Eq.~\eqref{eqn:deltaxdef}, and that displacements $\delta\mathbf{x}_t = \mathbf{x}(t) - \mathbf{x}^*(t)$ from the deterministic state $\mathbf{x}^*(t)$ scale as $\epsilon$, so that the linearized result Eq.~\eqref{eq:ICestimation} remains valid\@.
\end{assum}

% NOTE: Eliminate this proposition, state that the results directly hold by a receding-horizon argument
% NOTE: State that sum of the min-eigenvalues of the full LOG is intractable, but adopting the STLOG solves this issue
\begin{prop}\label{prop:STLOGupdate}
    Under Assumptions~\ref{assum:piecewiseConstantInput} and~\ref{assum:linearization}, the prior and posterior covariances at step \(k\) are related by:
    \begin{equation}\label{eq:kalmanlike1stlog}
        \hat{\mathbf{P}}_{T_k}^{-1} = \frac{\mathbf{W}^{(r)}_{\boldsymbol{\Sigma}^{-1}}\left( \mathbf{x}^*(T_k), \mathbf{u}(T_k) \right)}{\delta t}  + \check{\mathbf{P}}_{T_k}^{-1},
    \end{equation}
    where $\check{\mathbf{P}}_{T_k}$ is the covariance at time $T_k$ given only past observations, and $\hat{\mathbf{P}}_{T_k}$ is the covariance at time $T_k$ given additional observations between $T_k$ and \(T_{k+1}\), and \(\mathbf{W}^{(r)}_{\boldsymbol{\Sigma}^{-1}}\) is the r\textsuperscript{th}-order STLOG, with observation time \(\Delta T\)\@.
    In particular, an analogous bound to Eq.~\eqref{eq:icbound} holds: The posterior precision due to the observations in $\halfopen{T_k}{T_{k+1}}$ is bounded below by
    \begin{equation}\label{eq:Tkbound}
        {\lambda_{\max}^{-1}\left(\hat{\mathbf{P}}_{T_k}\right)} \geq \frac{\lambda_{\min}\left(\mathbf{W}^{(r)}_{\boldsymbol{\Sigma}^{-1}}\left( \mathbf{x}^*(T_k), \mathbf{u}(T_k) \right)\right)}{\delta t} + {\lambda_{\max}^{-1} \left(\check{\mathbf{P}}_{T_k}\right)}.
    \end{equation}

    \begin{proof}
        By analogy to the argument used to derive Eq.~\eqref{eq:ICestimation}, we perform linearization about $\mathbf{x}^*(t) = \phi_t (\mathbf{x}^*_0)$.
        By treating each prior mean $\check{\mathbf{x}}(T_k)$ as an initial condition with stage-wise observation timestep $\Delta T$ and linearizing about $\mathbf{x}^*(T_k)$ in the small noise regime in Eq.~\eqref{eqn:deltaxdef}, we apply Eq.~\ref{eq:ICestimation} to get the following evolution laws for the prior and posterior covariances at each $T_k$:
        \begin{equation}\label{eq:kalmanlike1}
            \check{\mathbf{P}}_{T_{k+1}} = \boldsymbol{\Phi}_{T_k, T_{k+1}}\hat{\mathbf{P}}_{T_k}\boldsymbol{\Phi}_{T_k, T_{k+1}}^\top, \quad \hat{\mathbf{P}}_{T_k}^{-1}  = \frac{\mathbf{W}_{\boldsymbol{\Sigma}^{-1}}\left[\mathbf{x}^*(T_k), \mathbf{u}: \halfopen{T_k}{T_{k+1}} \rightarrow \mathcal{U} \right]}{\delta t}  + \check{\mathbf{P}}_{T_k}^{-1}.
        \end{equation}
        where $\boldsymbol{\Phi}_{T_k, T_{k+1}} = D\left(\boldsymbol{\phi}_{T_{k+1}} \boldsymbol{\phi}^{-1}_{T_{k}}\right)$ is the state transition map from time \(T_k\) to \(T_{k+1}\).\footnote{
            Similar evolution laws relating the estimation means can also be obtained, but we omit their derivation here since they do not influence the evolution of covariances once the linear approximation around $\mathbf{x}^*(t) = \phi_t (\mathbf{x}^*_0)$ is admitted, as in Kalman filtering theory for linear systems.
        }
        Substituting the STLOG approximation in Eq.~\eqref{eq:stlogExpr} into Eq.~\eqref{eq:kalmanlike1} leads to Eq.~\eqref{eq:kalmanlike1stlog}.

    \end{proof}
\end{prop}
% \begin{rem}
%     The covariance update equation Eq.~\eqref{eq:kalmanlike1stlog} becomes equivalent to the original EKF covariance update equation if we take $r = 0$ and $\Delta T /\delta t = 1$.
%     In this case, only a single observation is received in each stage and the effects of system dynamics are moot.
%     This is reflected in the order-$0$ STLOG, which is expressed as $\mathbf{W}^{(0)}_{\boldsymbol{\Sigma}^{-1}} = \Delta T\mathbf{H}^\top \boldsymbol{\Sigma}^{-1} \mathbf{H}$.
%     Substituting this expression transforms Eq.~\eqref{eq:kalmanlike1stlog} into the ubiquitous EKF covariance update law \(
%     \hat{\mathbf{P}}_{T_k}^{-1} = \mathbf{H}^\top \boldsymbol{\Sigma}^{-1} \mathbf{H} + \check{\mathbf{P}}_{T_k}^{-1}
%     \).
% \end{rem}
% Ideally, we would like to choose the inputs $\mathbf{u}(t)$ so that they maximize each ${\lambda_{\min}\left(\mathbf{W}^{(r)}_{\boldsymbol{\Sigma}^{-1}}\left( \mathbf{x}^*(T_k), \mathbf{u}(T_k) \right)\right)}$.
% However, each $\mathbf{x}^*(T_k)$ depends functionally on the past controls $\mathbf{u}: \halfopen{0}{ T_k} \rightarrow \mathcal{U}$, and maximizing at each $T_k$ leads to an overdetermined problem.
Our observability-aware control principle is to collectively maximize the lower bounds to precision improvements, embodied by the minimum eigenvalue of the STLOG in Eq.~\eqref{eq:Tkbound},
along a trajectory.
This corresponds to a safety-oriented, robustness-driven design: by maximizing the smallest eigenvalue, the controller prioritizes the least observable state direction, guaranteeing a minimum reduction in estimation uncertainty and preventing severe loss of observability that could lead to CLS divergence.
This is represented by the objective:
\begin{equation}
    V(\mathbf{x}^*(T_{0,\ldots,N-1}), \mathbf{u}(T_{0,\ldots,N-1})) = \sum_{k=0}^{N - 1} {\lambda_{\min}\left(\mathbf{W}^{(r)}_{\boldsymbol{\Sigma}^{-1}}\left( \mathbf{x}^*(T_k), \mathbf{u}(T_k) \right)\right)},\label{eq:obsAwareCtrl}
\end{equation}
% subject to state dynamics in Eq.~\eqref{eq:dynamicalSystem} yields an observability-optimal sequence of control inputs \(\mathbf{u}(t) = \mathbf{u}_k\) for \(t \in \halfopen{T_k}{T_{k+1}}\).

The \emph{observability-predictive controller (OPC)} optimizes \(V\) under a receding horizon framework with a horizon \(N\):
\begin{align}\label{eq:OPCformula}
    \min_{\mathbf{u}_{0,\ldots,N-1}} & \frac{1}{V(\mathbf{x}_k, \mathbf{u}_k) +c}                                                                 \\
    \st\                             & \begin{cases}
                                           \mathbf{x}_{k+1} = \mathbf{f}_\mathrm{d}(\mathbf{x}_{k}, \mathbf{u}_k),\ \mathbf{x}_{0} = \mathbf{x}_{cl}, \\
                                           u_{\min, i}      \leq \mathbf{u}_{k, i} \leq u_{\max, i},                                                  \\
                                           d_{\min}         \leq d_k \leq d_{\max}.
                                       \end{cases}
\end{align}
where \(\mathbf{x}_0=\mathbf{x}_{cl}\) is the feedback of the current state, consisting of the inter-drone position, orientation, and velocity, estimated by a CLS.\@ \(\mathbf{x}(T_k) = \mathbf{x}_k\) and \(\mathbf{u}(T_k) = \mathbf{u}_k\) are the states and input at time \(T_k\) along the trajectory to be optimized. Following the receding-horizon paradigm, only the first control input \(\mathbf{u}_0\), consisting of thrust and body rates, is applied to the system.
Lastly, \(c\) is a small regularization constant which we take to be \(c = 10^{-6}\).

Unlike classical formation control, Eq.~\eqref{eq:OPCformula} does not penalize deviations from a specific formation or predefined positions between the leader and follower.
Instead, at every step \(k\), the leader-follower distance \(d_k\) is restricted to lie within the bounds
\([d_{\min}, d_{\max}]\) to maintain ranging measurements and connectivity.
In practice, the OPC operates as a high-level receding-horizon controller that receives the current CLS state estimate \(\mathbf{x}_{cl}\) as feedback and solves for optimal thrust and body-rate commands included in \(\mathbf{u}_k\), which are tracked by quadrotor low-level attitude controllers.

\subsection{Satisfaction of local weak observability under OPC control}\label{sub:STLOGranktesting}

Next, we address whether the optimal trajectories produced by the OPC inherently satisfy local weak observability.
We recover the discrete rank-testing notion of observability by proposing that the asymptotics of the summand~$\lambda_{\min}\left(\mathbf{W}^{(r)}_{\boldsymbol{\Sigma}^{-1}}\left(\mathbf{x},\mathbf{u}; \Delta T\right)\right)$ in Eq.~\eqref{eq:obsAwareCtrl} as~$\Delta T \to 0$ \emph{detects} the local observability index~\(r_*(\mathbf{x},\mathbf{u})\) in Eq.~\eqref{eq:rstardef}.
\begin{prop}\label{prop:STLOGasymptotics}
    Considering a fixed operating point \((\mathbf{x},\mathbf{u})\),
    \begin{enumerate}
        \item If $r < r_*(\mathbf{x},\mathbf{u})$, then $\lambda_{\min}(\mathbf{W}^{(r)}_{\boldsymbol{\Sigma}^{-1}}(\mathbf{x},\mathbf{u}; \Delta T)) = 0$,
        \item If $r \geq r_*(\mathbf{x},\mathbf{u})$, then for all sufficiently small $\Delta T$,
              \begin{align}\label{eq:stlogasymptotics}
                  \lambda_{\min}(\mathbf{W}^{(r)}_{\boldsymbol{\Sigma}^{-1}}(\mathbf{x},\mathbf{u}; \Delta T)) = C(\mathbf{x},\mathbf{u}; \Delta T)\Delta T^{2r_* + 1},
              \end{align}
              where $C(\mathbf{x},\mathbf{u}; \Delta T)$ is bounded between two positive quantities $C_1(\mathbf{x},\mathbf{u}), C_2(\mathbf{x},\mathbf{u})$ independent from~$\Delta T$. %: $0 < C_1(\mathbf{x},\mathbf{u}) \leq C \leq C_2(\mathbf{x},\mathbf{u})$.
    \end{enumerate}

    \begin{proof}
        This proposition is proved in detail in Appendix~\ref{app:asymptotics}.
    \end{proof}

\end{prop}

Proposition~\ref{prop:STLOGasymptotics} suggests that in the OPC, stages yielding low $r_*(\mathbf{x},\mathbf{u})$ make asymptotically larger contributions to the objective compared to stages yielding high $r_*(\mathbf{x},\mathbf{u})$.
Thus, optimized trajectories drive the state \(\mathbf{x}\) to avoid regions with high $r_*$ and favor regions where $r_*$ attains the minimum value.
We derive a lower bound for this minimum $r_*$ in Proposition~\ref{prop:UTdynsys1}, and apply it to our quadrotor CLS in Corollary~\ref{cor:rstar_quadrotorCLS} to obtain $r_* \geq 5$, providing us a principled guideline for selecting \(r\) for the STLOG in the OPC\@.

This result also connects our OPC to previous work optimizing metrics of the observability matrix~\cite{sharma2014observability}.
\(C_1,C_2\) in Proposition~\ref{prop:STLOGasymptotics} are shown by Propositions~\ref{prop:STLOGUBdone} and~\ref{prop:STLOGLBdone} to be:
\begin{align}
    C_1  = \alpha \ \sigma^2_{\min}\left( {\boldsymbol{\mathcal{O}}^{(r_*)}} \right),\quad C_2  = \beta \ \sigma^2_{\min}\left( {\boldsymbol{\mathcal{O}}^{(r_*)}} \vert_{\ker {\boldsymbol{\mathcal{O}}^{(r_* - 1)}}}\right),
\end{align}
where~\(\sigma_{\min}\) denotes the minimum singular value,~\(\vert \) denotes the restriction to a subspace, and~\(\alpha, \beta\) are numerical constants.
We identify that a new metric of observability~\(C^{\prime}\), formed by taking a weighted average of~\(C_1\) and~\(C_2\), can be heuristically substituted into \(C\) in Eq.~\eqref{eq:stlogasymptotics} and the objective formed by summing this quantity over a trajectory:
\begin{align}\label{eq:toyOPCprinciple}
    \mathcal{V}(\mathbf{x}^*(T_{0,\ldots,N-1}), \mathbf{u}(T_{0,\ldots,N-1})) \simeq \sum_{k=0}^{N - 1} C^{\prime}\left( \mathbf{x}^*(T_k), \mathbf{u}(T_k) \right) \Delta T^{2r_*\left( \mathbf{x}^*(T_k), \mathbf{u}(T_k) \right) + 1},
\end{align}
has the same physical meaning as the objective in Eq.~\eqref{eq:obsAwareCtrl}, in that for the same value of~\(r_*\), stages with a higher metric of observability~\(C^{\prime}\) will contribute more to \(\mathcal{V}\), leading to optimal observability trajectories under the metric~\(C^{\prime}\).

Overall, we met our design goals, as the OPC maximizes the worst-case precision improvement along a trajectory.
We also answered the question raised in Section~\ref{sec:observabilityAnalysis}, showing that the OPC is compatible with the implementation of weak local observability, then used the result to tailor the selection of \(r\) in the OPC to our range-only quadrotor CLS\@.

\section{Simulation and Experimental Validation}\label{sec:experimental}

\begin{figure}[htbp]
    \begin{center}
        \begin{tikzpicture}[scale=1, transform shape, node distance=1cm and 2cm, every node/.style={align=center}, arrow/.style={->, >=latex, ultra thick}]

    \node [draw,fill=white] (flightControl) {Geometric tracking\\Controller};
    \node [draw,fill=white,right=2.5cm of flightControl] (leaderDynamics) {Quadrotor\\Dynamics};
    \node [draw,fill=white,below=0.5cm of flightControl] (ekf) {GPS-driven EKF};
    \coordinate [below=of ekf] (t1);
    \coordinate (t2) at ($(flightControl)!0.5!(leaderDynamics)$);
    \node [draw, fill=green] (interDrone) at (t2 |- t1) {Inter-drone\\communication};
    \coordinate [below=of interDrone] (t3);
    \node [draw,fill=Turquoise] (uwb) at (t3 -| leaderDynamics){UWB Sensor};
    \node [draw,fill=white,right=1cm of uwb] (followerDynamics) {Quadrotor\\Dynamics};
    \node [draw,fill=white,below=0.6cm of followerDynamics] (opc) {\textbf{OPC in Eq~\eqref{eq:OPCformula}}\\\textbf{(our contribution)}};
    \node [draw,fill=white] (ekfCL) at (opc -| ekf) {EKF-based CLS\\ From Ref.~\cite{go2024trajectory}};
    \begin{scope}[on background layer]
        \node [fill=blue, opacity=0.3, fit=(opc)(ekfCL)(followerDynamics)(uwb)(interDrone.south), label={below:The follower drone}] {};
        \node [fill=Peach, opacity=0.3, fit=(flightControl)(leaderDynamics)(ekf)(interDrone.north), label={above:The leader drone}] {};
    \end{scope}

    \draw[arrow, Turquoise] (ekf) -- node [left] {\(\hat{\mathbf{p}}_l, \hat{\mathbf{v}}_l\)}(flightControl);
    \draw[arrow, Turquoise] (leaderDynamics) -- (leaderDynamics|-ekf) -- node [above] {\(\mathbf{p}_l, \mathbf{v}_l\) } (ekf);
    \draw[arrow, Turquoise] (leaderDynamics|-ekf) -- node [right] {\(\mathbf{p}_l\)} (uwb);
    \draw[arrow, magenta] ($(flightControl.south east)!0.5!(flightControl.east)$) -| node [right] {\(\mathbf{u}_l\)} (interDrone);
    \draw[arrow, Turquoise] (flightControl) -- node[above, draw=none] {Thrust/rates \(\mathbf{u}_l\)} (leaderDynamics);
    \draw[arrow, magenta] (opc) -- node [right] {\(\mathbf{u}_f\) } (followerDynamics);
    \draw[arrow, magenta] (followerDynamics) -- node [above] {\(\mathbf{p}_f\)} (uwb);
    \draw[arrow, magenta] (followerDynamics.south west) -- ++ (-0.25, -0.25) -| node [above] {\(\mathbf{q}_f\)} ($(ekfCL.north)!0.5!(ekfCL.north east)$);
    \draw[arrow, black] (uwb)  -| node[above, near start] {Inter-drone\\distance \(d\) } (ekfCL);
    \draw[arrow, magenta] (interDrone) -| node [left, near end] {\(\hat{\mathbf{p}}_l\)\\\(\hat{\mathbf{q}}_l\)\\\(\mathbf{u}_l\)} ($(ekfCL.north)!0.5!(ekfCL.north west)$);
    \draw [arrow, magenta] ($(ekf.south west)!0.5!(ekf.south)$) |- node [left, near start] {\(\hat{\mathbf{p}}_l\)\\\(\hat{\mathbf{q}}_l\)} ($(interDrone.north west)!0.5!(interDrone.west)$);

    \draw[arrow, ultra thick, PineGreen] (opc) -- node [below] {Optimized thrust/rates \(\mathbf{u}_f\)} (ekfCL);

    \draw[arrow, magenta] ($(ekfCL.east)!0.5!(ekfCL.north east)$) -- ++ (1, 0) |- node [above] {\(\hat{\mathbf{x}}\)} ($(opc.west)!0.5!(opc.north west)$);
\end{tikzpicture}
    \end{center}
    \caption{The overall flow chart of the system.}%
    \label{fig: overall_flow_chart}
\end{figure}

We evaluate our OPC on CLS-aided quadrotors in two stages.
In Section~\ref{sub:simulations}, we describe our Monte Carlo simulations, modelled after our previous work~\cite{go2024trajectory}, of an EKF-based CLS configured according to Fig.~\ref{fig: overall_flow_chart}.
We evaluate the localization performance of this system when guided by our controller compared to two baseline trajectories.
In Section~\ref{sub:experiments}, we describe our quadrotor flight experiments using the same EKF-based CLS\@.

The computational cost of the OPC is dominated by two factors: evaluating the STLOG and solving the receding-horizon optimal control problem.
To evaluate the STLOG, we compute Lie derivatives by \emph{automatic differentiation} (via \texttt{jax}~\cite{jax2018github}\footnote{Python implementation of the OPC can be found in\url{https://github.com/FSC-Lab/observability_aware_controller}}),
{which is competitive with quadratic Lie-derivative approximations used in Ref.~\cite{yu2015observability} while avoiding truncation errors and simplifying implementation.}
Representative evaluation and solve times are summarized in Table~\ref{tab:stlogEvalTime}.

\begin{table}[htbp]
    \centering
    \caption{STLOG evaluation time}\label{tab:stlogEvalTime}
    \begin{tabularx}{0.8\linewidth}{Xccc} \toprule
        Component            & Setting                     & Typical time & Worst-case time \\ \midrule
        STLOG evaluation     & $r=5$                       & 0.05--0.15 s & $<$ 0.3 s       \\
        Objective evaluation & $r=5$, $N=20$               & 0.1--0.3 s   & $<$ 1.0 s       \\
        Total OPC solve      & $r=5$, $N=20$, 40 iters max & 2.0--4.0 s   & $<$ 6.0 s       \\ \bottomrule
    \end{tabularx}
\end{table}

The optimization cost scales primarily with the horizon length \(N\) and the number of solver iterations.
With the settings reported in Table~\ref{tab:stlogEvalTime}, solve times are compatible with online, receding-horizon execution at guidance-level update rates, but are not yet suitable for real-time thrust and body-rate generation on embedded hardware.
Therefore, online execution of the OPC is evaluated in simulation, while in experiments, the OPC is used as an offline trajectory generator.

\begin{figure}[htb]
    \centering
    \includegraphics{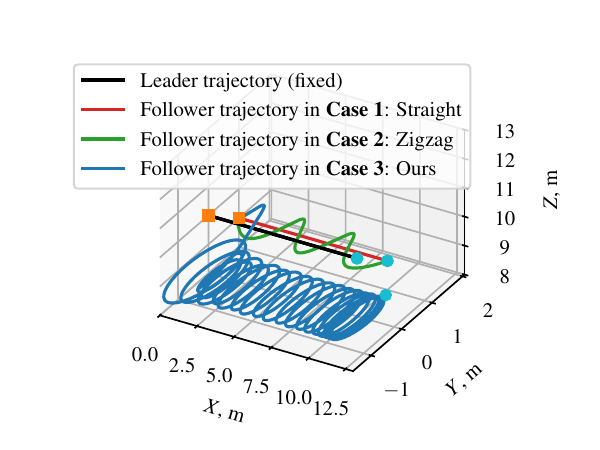}%
    \caption{Simulated trajectories with squares/circles marking starting/end positions, respectively.}%
    \label{fig:ekfSimTrajectory}
\end{figure}

\subsection{Simulation}\label{sub:simulations}

In simulations, the leader follows a straight trajectory, while the follower performs three test cases:
\begin{description}[itemindent=!]
    \item[Case 1] A uniform and straight trajectory, planned ahead-of-time by a minimum-snap trajectory planner,
    \item[Case 2] An artificially perturbed, zigzagging trajectory, likewise planned ahead-of-time, and
    \item[Case 3] An observability-aware trajectory performed with the OPC computing control inputs online.
\end{description}

We use a range-only CLS based on an EKF as configured in Table~\ref{tab:simCLSCfg} to estimate the states of both quadrotors, so that estimation performance can be compared across the three test cases.
We implemented realistic noise models for both control inputs and observations received by the CLS\@.
The simulations are run for 120 seconds, with 50 Monte Carlo trials per test case.
The same noise regime is applied to all trials.

\begin{table}[htbp]
    \caption{Configuration of CLS in simulation}\label{tab:simCLSCfg}
    \begin{center}
        \begin{tabularx}{0.8\linewidth}{Xccc} \toprule
            Noise Type                                  & Parameter                                               & Value                                     & Notes                    \\ \midrule
            \multirow{2}{*}{\textbf{Process noise}}     & \(\sigma^2\left(f^{\{f, l\}}\right)\)                   & \(0.05\  (\mathrm{unitless})\)            & Thrust model uncertainty \\
                                                        & \(\sigma^2\left(\boldsymbol{\omega}^{\{f, l\}}\right)\) & \(1.49\cdot 10^{-5}\ \mathrm{rad}^2/s^2\)                            \\\midrule
            \multirow{2}{*}{\textbf{Observation Noise}} & \(\sigma^2\left(r\right)\)                              & \(0.008\ \mathrm{m}^2\)                   & TOA mode                 \\
                                                        & \(\sigma^2\left(\boldsymbol{\phi}\right)\)              & \(3.594 \cdot 10^{-6}\ \mathrm{rad}^2\)   & Per-rotation axis        \\
            \bottomrule
        \end{tabularx}
    \end{center}
\end{table}

Figure~\ref{fig:ekfSimTrajectory} shows the trajectories trialed in Monte Carlo simulations.
Our optimized trajectories in Case 3 exhibit distinct behaviors to deliberately excite weakly observable directions.
This periodic dynamic formation emerges from the OPC without specific hints or deliberate configuration, reproducing the helical~\cite{li2022three} and orbiting~\cite{boyinine2022observability} behavior of observability-optimal trajectories in previous works.
Furthermore, our optimized trajectories implied moderate accelerations, as summarized in Table~\ref{tab:simCLSAccels}, compatible with quadrotor dynamics and control.

\begin{table}[htb]
    \centering
    \caption{Acceleration implied by different trajectories in simulations}\label{tab:simCLSAccels}
    \begin{tabularx}{\linewidth}{Xccc} \toprule
        Trajectory    & Peak X Acceleration, \(m/s^2\) & Peak Y Acceleration, \(m/s^2\) & Peak Z Acceleration, \(m/s^2\) \\ \midrule
        Case 3 (ours) & 1.737                          & 2.063                          & 1.644                          \\ \bottomrule
    \end{tabularx}
\end{table}

\begin{figure}[htbp]
    \centering
    \subfloat[X-axis position error; Horizontal lines show RMS error\label{fig:ekfSimXResults}.]{\includegraphics{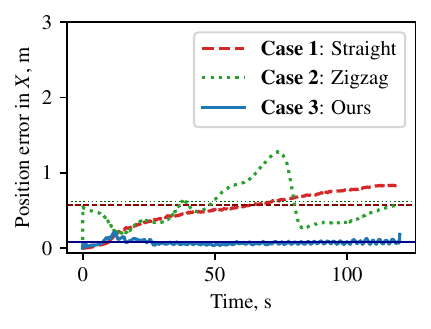}}%
    \subfloat[Estimated X-axis uncertainty (\(3\sigma\) envelope).]{\includegraphics{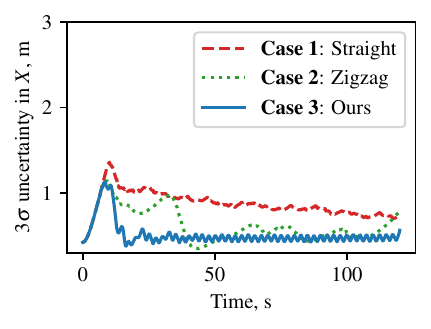}}

    \subfloat[Y-axis position error; Horizontal lines show RMS error.]{\includegraphics{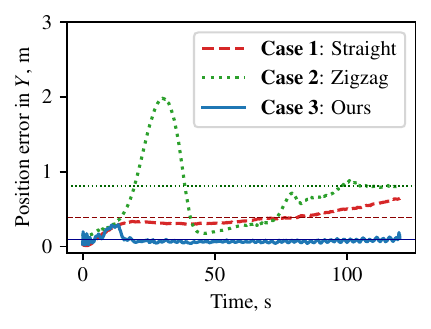}}%
    \subfloat[Estimated Y-axis uncertainty (\(3\sigma\) envelope).]{\includegraphics{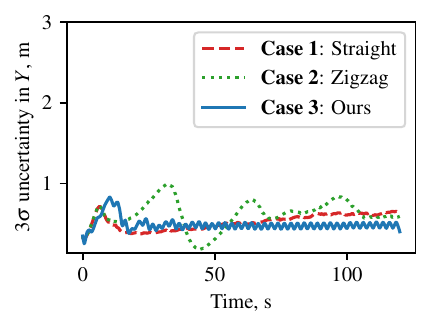}}

    \subfloat[Z-axis position error; Horizontal lines show RMS error.]{\includegraphics{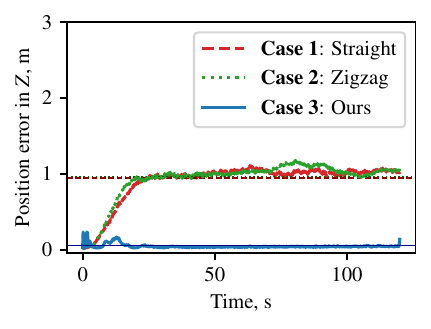}}%
    \subfloat[Estimated Z-axis uncertainty (\(3\sigma\) envelope)\label{fig:ekfSimZVariance}.]{\includegraphics{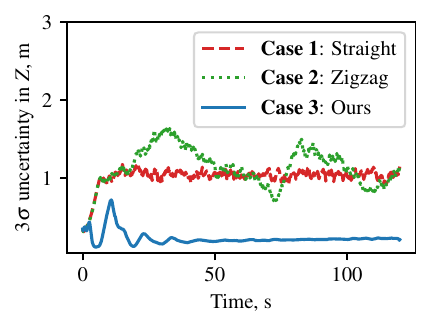}}
    \caption{CLS performance in simulations, quantified by per-axis estimation error, RMSE, and uncertainty.}%

\end{figure}

Figures~\ref{fig:ekfSimXResults} to~\ref{fig:ekfSimZVariance} show the CLS\@ performance in simulations.
The straight trajectory in Case 1 yields a steadily diverging error and high covariance due to a lack of observability-aware excitation, in line with long-standing conclusions about the pathologies of uniform formations~\cite{trawny2004optimized}.
The zigzag trajectory in Case 2 is dynamically rich but estimation-unaware.
It enabled the CLS to occasionally reduce position estimation error at the cost of strong spikes in error that drove up the RMS error.
This confirms that simply exciting position space or desynchronizing follower motion relative to the leader does not guarantee informative sensor observations.
In contrast, the observability-optimal dynamic formation explicitly excites low-observability directions and maintains low positioning error and tight estimator confidence bounds.

\begin{table}[htb]
    \centering
    \caption{Positioning accuracy of EKF-based CLS along different trajectories in simulations}\label{tab:simCLSResults}
    \begin{tabularx}{\linewidth}[c]{Xccc ccc ccc} \toprule
        \multirow{2}{*}{Trajectory} & \multicolumn{3}{c}{X Positioning Error, m} & \multicolumn{3}{c}{Y Positioning Error, m} & \multicolumn{3}{c}{Z Positioning Error, m}                                                                                                                 \\\cmidrule{2-10}
                                    & min                                        & max                                        & RMS                                        & min              & max             & RMS              & min              & max             & RMS              \\\midrule
        Case 1 (straight)           & 7.892e-4                                   & 0.8381                                     & 0.5769                                     & 0.01265          & 0.6452          & 0.3891           & 0.01802          & 1.109           & 0.942            \\
        Case 2 (zigzag)             & 0.1959                                     & 1.284                                      & 0.6168                                     & 0.1419           & 1.983           & 0.8118           & 0.02051          & 1.179           & 0.9636           \\
        Case 3 (ours)               & \textbf{0.004418}                          & \textbf{0.2121}                            & \textbf{0.07992}                           & \textbf{0.02308} & \textbf{0.2953} & \textbf{0.09616} & \textbf{0.02199} & \textbf{0.2339} & \textbf{0.05407} \\\bottomrule
    \end{tabularx}
\end{table}

\begin{table}[htb]
    \centering
    \caption{Positioning uncertainty of EKF-based CLS along different trajectories in simulations}\label{tab:simCLSUncertainty}
    \begin{tabularx}{\linewidth}[c]{Xccc} \toprule
        Trajectory        & X area of 3\(\sigma\) envelope, \(m\cdot s\) & Y Area of 3\(\sigma\) envelope, \(m\cdot s\) & Z Area of 3\(\sigma\) envelope, \(m\cdot s\) \\\midrule
        Case 1 (straight) & 104.789                                      & 62.469                                       & 121.943                                      \\
        Case 2 (zigzag)   & 74.959                                       & 73.587                                       & 135.399                                      \\
        Case 3 (ours)     & \textbf{61.466}                              & \textbf{58.792}                              & \textbf{27.316}                              \\\bottomrule
    \end{tabularx}
\end{table}

Quantitatively, Table~\ref{tab:simCLSResults} shows that the optimized trajectory computed by our OPC in Case 3 yields the lowest position estimation error on all axes and consistently reduces the position estimation RMS error by more than 75\% compared to both Cases 1 and 2.
Table~\ref{tab:simCLSUncertainty} shows that the same trajectory yields the lowest 3\(\sigma\) confidence envelope area along every axis, indicating that it provides the most precise position estimates.
These figures confirm that observability-aware control strategies not only reduce average estimation error but also improve estimator consistency and increase precision.

\subsection{Quadrotor Flight Experiments}%
\label{sub:experiments}

\begin{figure}[htb]
    \centering
    \includegraphics[width=0.3\linewidth]{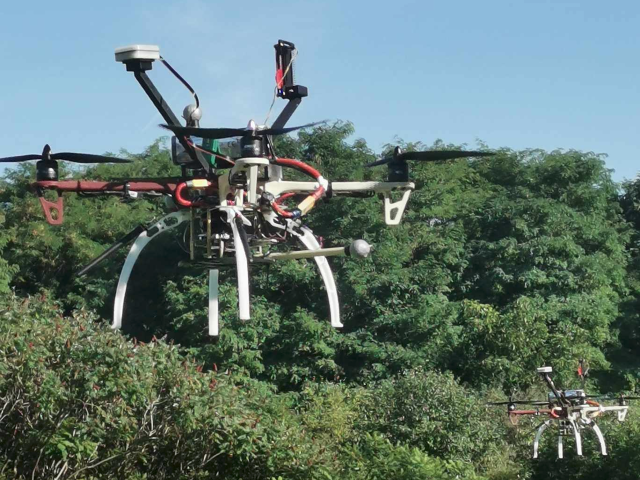}
    \caption{A UWB-ranging based CLS on a pair of F450 quadrotors during flight tests\label{fig:realSystemDepiction}}
\end{figure}

\begin{table}[htb]
    \centering
    \caption{Configuration of F450 Experimental Platforms}%
    \label{tab:quadrotor}

    \begin{threeparttable}
        \begin{tabularx}{\linewidth}{cXcc} \toprule
            Hardware            & Name and Specifications            & Software          & Name/Specifications                              \\ \midrule
            Airframe            & F450 (1.7kg with LiPo, 10`` props) & SDK Framework     & \lstinline!mavros!                               \\
            Autopilot           & Cube Orange (PX4 1.14.5)           & Setpoint Type     & \lstinline!mavros_msgs::AttitudeTarget!\tnote{1} \\
            Onboard Computer    & Jetson TX2 NX (JetPack 4.6)        & Setpoint Rate     & Upsampled to 50 Hz                               \\
            GNSS unit (leader)  & UBlox M8N                          & Leader Controller & Geometric tracking controller                    \\
            UWB unit (follower) & ESP32 UWB DWM3000\tnote{2}         & Leader Estimator  & GPS-driven EKF                                   \\ \bottomrule
        \end{tabularx}
        \begin{tablenotes}
            \footnotesize
            \item[1] Approximately 0.12s settling time in response to step inputs in body-rate control.
            \item[2] The C/C++ interface to the UWB sensor can be found at: \url{https://github.com/FSC_Lab/dwm3000_ros}
        \end{tablenotes}
    \end{threeparttable}
\end{table}

Next, we apply our OPC to formation flight experiments with quadrotors depicted in Fig.~\ref{fig:realSystemDepiction} and configured according to Table~\ref{tab:quadrotor}.
The leader follows a straight trajectory, while the follower performs two test cases, excluding the zigzagging trajectory (Case 2 in simulations), as it contributed to worse positioning in simulations.
\begin{description}[itemindent=!]
    \item[Case 1] A uniform and straight trajectory, planned ahead-of-time by a minimum-snap trajectory planner,
    \item[Case 2] An observability-aware trajectory, planned ahead-of-time by running the OPC as a sliding-window optimizer.
\end{description}
We use a range-only CLS interfaced with a UWB range system to estimate the states of both quadrotors, so that the estimation performance can be compared between the two test cases.

Figures~\ref{fig:straightTrajectory} and~\ref{fig:optimalTrajectories} show straight (Case 1) and observability-aware (Case 2) trajectories during flight tests.
The observability-aware trajectory induces an orbiting motion in the follower, while the leader continues along a uniform straight path.
This motion pattern yields significantly more excitation in the inter-vehicle range.
% Figure~\ref{fig:uwbMeas} shows the green inter-vehicle distance oscillates around a tighter band compared to the steadily growing profile of the straight trajectory, corroborating the analytically derived pattern~\cite{boyinine2022observability}, where the periodic variation of the range yields optimal observability.

\begin{figure}[!ht]
    \centering
    \subfloat[Quadrotor trajectories in flight experiment Case 1.\label{fig:straightTrajectory}.]{\includegraphics[]{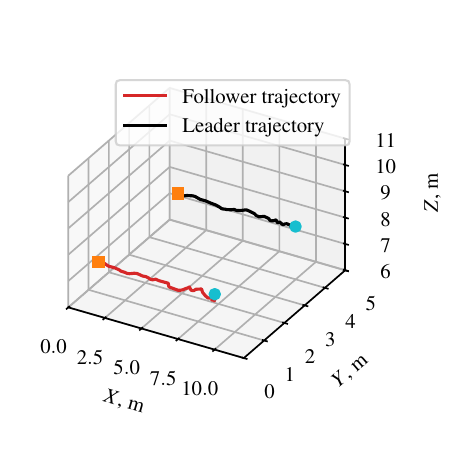}}
    \subfloat[Quadrotor trajectories in flight experiment Case 2.\label{fig:optimalTrajectories}.]{\includegraphics[]{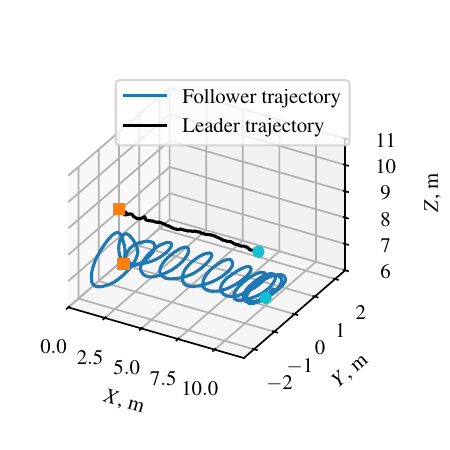}}
    \hskip10pt
    \caption{Experimental trajectories with squares/circles marking starting/end positions, respectively.}%

\end{figure}

Figure~\ref{fig:flightResults} shows the CLS\@ positioning performance in flight tests.
The key performance criterion is the estimation precision: the area of the 3\(\sigma\) confidence envelope is consistently smaller for the optimized trajectory compared to the straight trajectory.
Quantitatively, Table~\ref{tab:flightCLSUncertainty} shows that optimized trajectories yielded more than 50\% better overall positioning uncertainty along every axis.
Furthermore, the maximum 3\(\sigma\) levels never exceeded 2 m for the optimized trajectory, while those of the straight trajectory exceeded 5 m for significant durations.

\begin{table}[!ht]
    \caption{CLS performance in flight, quantified by per-axis estimation error, RMSE, and uncertainty}\label{tab:flightCLSUncertainty}
    \begin{tabularx}{\linewidth}{Xccc} \toprule
        Trajectory        & X area of 3\(\sigma\) envelope, \(m\cdot s\) & Y Area of 3\(\sigma\) envelope, \(m\cdot s\) & Z Area of 3\(\sigma\) envelope, \(m\cdot s\) \\\midrule
        Case 1 (straight) & 119.0                                        & 122.6                                        & 153.5                                        \\
        Case 2 (ours)     & \textbf{53.97}                               & \textbf{49.61}                               & \textbf{59.81}                               \\\bottomrule
    \end{tabularx}
\end{table}

In the absence of ground-truth data, the positioning error in Fig.~\ref{fig:flightResults} is computed relative to a GNSS reference post-processed by an EKF, which gave a variance of approximately 0.3\(m^2\) reflecting the post-fit dispersion of the filtered solution.
Although the positioning error should be interpreted as a relative and qualitative indicator, rather than an absolute measure of accuracy, the results remain informative.
Our optimized trajectory in Case 2 maintained a bounded positioning error evenly on both X and Y axes.
In contrast, the straight trajectory in Case 1 exhibited asymmetric performance: better accuracy on the X-axis but significantly worse on the Y-axis, reflecting poor excitation in that dimension.
This shows that the optimized trajectory may contribute to an even distribution of positioning performance across all axes, which is consistent with our goal of mitigating potentially dangerous worst-case positioning errors.

Together, our simulations and flight experiments demonstrate that our OPC can effectively improve the accuracy and precision of state estimation for the range-only quadrotor CLS.
This improvement is achieved in the presence of realistic noise in simulations and practical model uncertainties in flight tests, even though our observability cost was derived assuming noiseless state evolution.
This validates our approach and reinforces the utility of the OPC in applied settings.

\begin{figure}
    \centering
    \subfloat[X-axis position error.]{\includegraphics[]{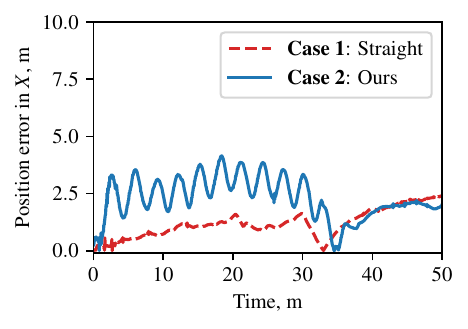}}
    \subfloat[Estimated X-axis uncertainty (\(3\sigma\) envelope).]{\includegraphics[]{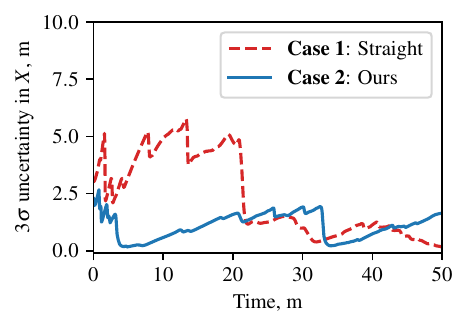}}

    \subfloat[Y-axis position error.]{\includegraphics[]{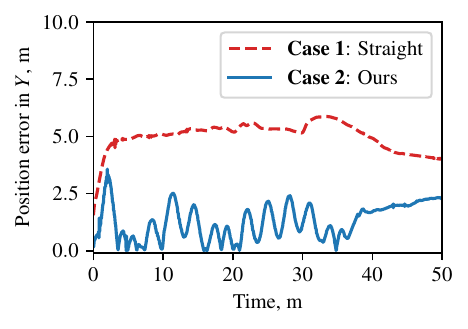}}
    \subfloat[Estimated Y-axis uncertainty (\(3\sigma\) envelope).]{\includegraphics[]{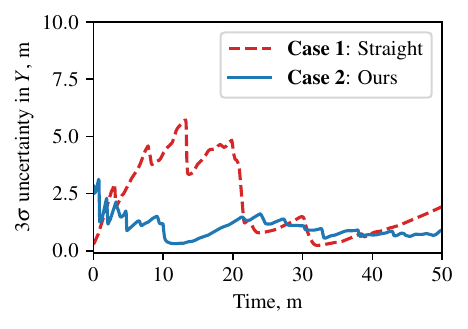}}

    \subfloat[Z-axis position error.]{\includegraphics[]{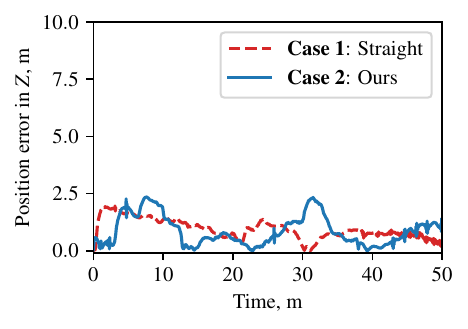}}
    \subfloat[Estimated Z-axis uncertainty (\(3\sigma\) envelope).]{\includegraphics[]{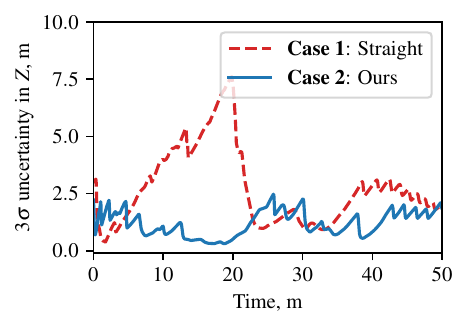}}

    \caption{CLS positioning performance in flight, with squares/circles marking starting/end positions, respectively.}%
    \label{fig:flightResults}
\end{figure}

\section{Conclusion}\label{sec:conclusion}

This work demonstrates that observability-aware control enables robust range-only cooperative localization systems (CLS) for quadrotors.
We addressed the fragility of range-only relative localization, where specific state components become poorly observable without sufficiently rich motion.

The proposed STLOG builds upon the classical LOG, which captures information lost by first-order quantities like the FIM, and enables tractable, point-wise evaluation of observability along a trajectory.
The resulting OPC serves as both a robust controller and safety mechanism: by maximizing the minimum eigenvalue of the STLOG, it generates motions that excite poorly observable directions, guarantees a minimum precision improvement in state estimation over each control stage, and prevents catastrophic uncertainty growth and estimator divergence.

Our simulations and experiments confirm that observability-aware dynamic formations outperform rigid and perturbed baselines in terms of estimation precision.
These results validated the proposed framework for application to quadrotor teams, specifically for leader-follower ferrying missions in GNSS-denied environments.
Our leader-follower configuration is limited to two quadrotors, for which a centralized cooperative localization and control architecture is sufficient for simplicity, but our inter-drone dynamics model can be generalized beyond quadrotors to pairs of multirotors in general.
The richness of this model incurs a high computational cost in evaluating its observability, which currently restricts the OPC to an offline guidance planner role in experiments.
We leave the extension of our approach to distributed estimation and control, and addressing computational bottlenecks, to future work.

\titleformat{\section}
{\normalfont\large\bfseries\centering} % Format of the text (bold, large, etc.)
{\centering Appendix \thesection: #1}      % The label (e.g., "Appendix A:")
{1em}                        % Space between label and title
{}
\begin{appendices}

    \section{Derivation of closed form posterior uncertainty}\label{app:CameronMartin}

This appendix supports the derivation of the likelihood ratio \(p\left(\frac{\mathbf{y}(t)|\mathbf{x}_0}{\mathbf{y}(t)}\right)\) proposed in Eq.~\eqref{eqn:likelihoodratio1} within Sec.~\ref{sec:measurementOfObservability} to quantify posterior uncertainty through continuous-time measurements.
Likelihood ratios in this form are \emph{Radon--Nikodym densities} between probability measures on an infinite-dimensional
path space.
In the Gaussian setting, they admit a closed form via the Cameron--Martin theorem.
We refer to~\cite{pollard2002user} for background on measure-theoretic probability, and to~\cite{bogachev2015gaussian,hairer2023introductionstochasticpdes} on Gaussian measure theory.

We recall the Cameron--Martin theorem, together with the definitions of the Reproducing Kernel Hilbert Space (RKHS), the Cameron--Martin space, and the associated Riesz representations.
The Cameron--Martin theorem characterizes when two Gaussian measures with the same covariance are absolutely continuous w.r.t.\ each other.

Let $H$ be a separable Hilbert space with inner product $\langle \cdot,\cdot \rangle_H$.
Let \(\mu\) be a \emph{centered Gaussian measure} on \(H\), meaning that for every continuous linear functional $l \in H^*$, the pushforward measure $l_* \mu$ on $\mathbb{R}$ is a one-dimensional Gaussian measure with zero mean.
The mean $m: H^* \rightarrow \mathbb{R}$ and covariance $\mathcal{K}: H^* \times H^* \rightarrow \mathbb{R}$ of $\mu$ are defined as the continuous linear functionals \(m(l) = \int_H l(y) \ d\mu(y)\) and \(\mathcal{K}(l, l^\prime)  = \int_H l(y-m) l^\prime(y-m) \ d\mu(y)\) for all $l, l^\prime \in H^*$.

In the following, we focus on the covariance operator $\mathcal{K}: H \rightarrow H$.
By the Riesz identification \(H\simeq H^*\), we use the same symbol \(\mathcal{K}\) for the induced bilinear form on \(H^*\times H^*\).
We assume $\mathcal{K}$ is a bounded, self-adjoint, and trace-class operator.
Then, there exists a centered Gaussian measure $\mu_{0,\mathcal{K}}$ with covariance \(\mathcal{K}\).
All subsequent constructions are expressed in terms of $\mathcal{K}$.
\begin{defn}[Reproducing kernel Hilbert space]
    Let $\iota: H^* \to L^2(H,\mu_{0,\mathcal{K}})$ denote the canonical embedding that identifies linear functionals with square-integrable functions.
    The closure of $\iota(H^*)$ in $L^2(H, \mu_{0,\mathcal{K}})$ under the Hilbert space norm is the \emph{reproducing kernel Hilbert space (RKHS)} $\mathcal{R}_{\mathcal{K}}$, and is itself a Hilbert space with the inner product \(\langle \iota(l), \iota(l^\prime) \rangle = \int_H l(y)l^\prime(y) \ d\mu_{0,\mathcal{K}}(y)= \mathcal{K}(l, l^\prime)\).
\end{defn}

\begin{defn}[Cameron--Martin space]
    The Cameron--Martin space $\mathcal{H}_{\mathcal{K}}$ is the subset of $H$ defined by \(\mathcal{H}_{\mathcal{K}} = \left\{ h \in H : \sup \{l(h) : l \in H^*, \mathcal{K}(l,l) \leq 1 \} < \infty \right\}\).
\end{defn}

\begin{defn}[Riesz representation in the Cameron--Martin space]
    For $h \in \mathcal{H}_{\mathcal{K}}$, the Riesz representation $\hat{h} \in \mathcal{R}_{\mathcal{K}}$ is defined by:
    \begin{align}\label{eqn:CM_iso_RKHS}
        l(h) = \int_H l(y)\hat{h}(y) \ d \mu_{0,\mathcal{K}}(y) \quad \text{for all $l \in H^*$},
    \end{align}
    and $h \mapsto \hat{h}$ is an isomorphism of Hilbert spaces $\mathcal{H}_{\mathcal{K}} \cong \mathcal{R}_{\mathcal{K}}$ when the Cameron--Martin space is equipped with the following inner product:
    \begin{align}\label{eqn:CM_inner_product}
        \langle h, k \rangle_{\mathcal{H}_{\mathcal{K}}} = \langle \hat{h}, \hat{k} \rangle_{\mathcal{R}_{\mathcal{K}}}.
    \end{align}
\end{defn}

We recall the Cameron–Martin theorem below.
\begin{thm}[Cameron--Martin]
    The Gaussian measure $\mu_{h,\mathcal{K}}$ is absolutely continuous with respect to $\mu_{0,\mathcal{K}}$, if and only if $h \in \mathcal{H}_\mathcal{K}$.
    In this case, the Radon--Nikodym density is:
    \begin{align}\label{eqn:CMtheorem1}
        \frac{d \mu_{h,\mathcal{K}}}{d \mu_{0,\mathcal{K}}}(y) = \exp\left( \hat{h}(y) - \frac{1}{2} \langle h, h \rangle_{\mathcal{H}_\mathcal{K}} \right),
    \end{align}
\end{thm}

Let ${\{e_i\}}_{i=1}^\infty$ be an orthonormal eigenbasis for $\mathcal{K}$ with corresponding eigenvalues $\{\lambda_i\}_{i=1}^\infty$.
If $\lambda_i > 0$ for all $i$, the following characterization holds:
\begin{enumerate}
    \item $\mathcal{H}_{\mathcal{K}}$ is the set of elements $h$ satisfying $\sum_{i=1}^\infty \langle h, e_i \rangle^2/\lambda_i < \infty$;
    \item The inner product in Eq.~\eqref{eqn:CM_inner_product} is given explicitly as $\langle h, k \rangle_{\mathcal{H}_{\mathcal{K}}} = \langle \mathcal{K}^{-1/2} h, \mathcal{K}^{-1/2} k \rangle_{H}$;
    \item The map $h \mapsto \hat{h}$ in Eq.~\eqref{eqn:CM_iso_RKHS} is given explicitly as $\hat{h} = \mathcal{K}^{-1}h$.
\end{enumerate}

\textbf{To obtain a likelihood ratio that depends continuously on both the observation trajectory and the finite-dimensional quantity \(\mathbf{x}_0\), we focus on the case where $\mathcal{K}$ is positive definite, and only consider mean displacements in a smaller subspace $\mathcal{H}_{\mathcal{K}^2} \subset \mathcal{H}_\mathcal{K}$.}
In the context of Sec.~\ref{sec:measurementOfObservability}, \(\mathbf{x}_0\in \mathcal{X}\) corresponds to the initial condition, while \(\mathcal{A}(\mathbf{x})\) denotes the induced observation trajectory.
Since $\mathcal{H}_{\mathcal{K}^2} \hookrightarrow \mathcal{H}_\mathcal{K} \hookrightarrow H$ and $\mathcal{H}_{\mathcal{K}^2} \xrightarrow[]{\mathcal{\mathcal{K}}^{-1/2}} \mathcal{H}_\mathcal{K} \xrightarrow[]{\mathcal{K}^{-1/2}} H$ are continuous, we have:
\begin{cor}
    Let $\mathcal{K}$ be a positive definite trace-class operator on a separable Hilbert space $H$, and let $\mathcal{H}_{\mathcal{K}^2}$ be the Cameron--Martin space of $\mathcal{K}^2$.
    For $h \in \mathcal{H}_{\mathcal{K}^2}$, the Radon--Nikodym density of $\mu_{h,\mathcal{K}}$ with respect to $\mu_{0,\mathcal{K}}$ exists, and moreover is a continuous function on $\mathcal{H}_{K^2} \times H$ given by:
    \begin{align}\label{eqn:CMtheorem2}
        \frac{d \mu_{h,\mathcal{K}}}{d \mu_{0,\mathcal{K}}}(y) = F(h, y) = \exp\left( \langle y, \mathcal{K}^{-1}h\rangle_H - \frac{1}{2}\langle h, \mathcal{\mathcal{K}}^{-1}h \rangle_H \right).
    \end{align}
\end{cor}

Now consider a finite-dimensional space $X \subset \mathbb{R}^n$ and a continuous function $\mathcal{A}: X \rightarrow \mathcal{H}_{\mathcal{K}^2}$, and also
let $\nu$ be a Borel probability measure on $X$.
The space $X$ should be considered as a finite-dimensional parametrization of possible means for the Gaussian measure on $H$, and $\nu$ our uncertainty of this parameter.
We define the joint probability distribution on $X\times H$ to be the measure
\begin{align}\label{def:jointmeasure1}
    \gamma(x,y) = F(\mathcal{A}(x), y) (\nu \otimes \mu_{0,\mathcal{K}})(x, y),
\end{align}
where $F$ is the positive function given by Eq.~\eqref{eqn:CMtheorem2} and \(\otimes\) denotes the product measure.
Since $(x,y) \mapsto F(\mathcal{A}(x), y)$ is continuous, it is Borel measurable and $\nu \otimes \mu_{0,\mathcal{K}}$-integrable.

By Tonelli's theorem, the following normalization ensures that \(\gamma\) defines a valid joint distribution compatible with Bayes’ rule in Eq.~\eqref{eqn:likelihoodratio1}.
\begin{align*}
    \int_{X \times H} F(\mathcal{A}(x), y) \ d (\nu \otimes \mu_{0,\mathcal{K}})(x, y) = & \int_{X}\left( \int_H F(\mathcal{A}(x), y) \ d\mu_{0,\mathcal{K}}(y)\right) d\nu(x) = \int_{X}\left( \int_H d\mu_{\mathcal{A}(x),\mathcal{K}}(y)\right) d\nu(x) \\
    =                                                                                    & \int_{X} 1 \ d\nu = 1,
\end{align*}

Using the Cameron--Martin formula in Eq.~\eqref{eqn:CMtheorem1} in the penultimate step, we verify that the conditional measure $\gamma|_x$ of $\gamma$ in Eq.~\eqref{def:jointmeasure1} at a given $x \in X$ and the marginal measure $\bar{\gamma}$, both measures on $H$, are given respectively by:
\begin{align}
    \gamma|_x (y) = F(\mathcal{A}(x), y) \mu_{0,\mathcal{K}}(y) = \mu_{\mathcal{A}(x), \mathcal{K}}(y), \quad \bar{\gamma} (y) = \left(\int_X F(\mathcal{A}(x), y) \ d\nu(x) \right) \mu_{0,\mathcal{K}}(y),
\end{align}
from which it follows that both $\gamma|_x$ and $\bar{\gamma}$ have positive Radon--Nikodym densities with respect to $\mu_{0,\mathcal{K}}$, and therefore are absolutely continuous with respect to each other.
The following result provides the likelihood ratio between the conditional and marginal observation measures required by Eq.~\eqref{eqn:likelihoodratio1}.

\begin{prop}\label{prop:main_likelihood_ratio}
    Consider a finite-dimensional parameter space $X \subset \mathbb{R}^n$ and a Borel measure $\nu$ on $X$.
    Let $\mu_{\mathcal{A}(x), \mathcal{K}}$ be a family of Gaussian measures on a Hilbert space $H$, with the same positive-definite trace-class covariance $\mathcal{K}$, and mean parametrized by a continuous map $\mathcal{A}: X \rightarrow \mathcal{H}_{\mathcal{K}^2}$.
    Consider the following two distributions on $H$:
    \begin{enumerate}
        \item $\gamma|_x = \mu_{\mathcal{A}(x), \mathcal{K}}$, the conditional distribution of $y\in H$ for a given $x\in X$,
        \item $\bar{\gamma}$, the marginal distribution of $y\in H$ obtained from the marginalizing the joint distribution on $X \times H$,
    \end{enumerate}
    then $\gamma|_x$ and $\bar{\gamma}$ are absolutely continuous with respect to each other, with Radon--Nikodym density
    \begin{align}
        \frac{d\gamma|_x }{d\bar{\gamma} }(y) & = C(y) F(\mathcal{A}(x), y),\quad C(y) = \left(\int_X F(\mathcal{A}(x), y) \ d\nu(x) \right)^{-1},
    \end{align}
    where $F$ is given by Eq.~\eqref{eqn:CMtheorem2} and $C(y)$ is independent of $x\in X$.
\end{prop}

A useful special case is when $H = L^2([0,T], \mathbb{R}^p)$, the Hilbert space of square-integrable $\mathbb{R}^p$-valued functions on the closed interval $[0,T]$.
Our convention for the Hilbert space inner product is
\begin{align}\label{eqn:InnerProductDef}
    \left\langle\mathbf{f}(t), \mathbf{g}(t) \right\rangle_{L^2([0,T], \mathbb{R}^p)} = \frac{1}{T} \int_0^T \mathbf{f}^\top(t) \mathbf{g}(t) \ dt.
\end{align}

To obtain sufficient conditions in the continuous-time trajectory space used in this work, we specialize Proposition~\ref{prop:main_likelihood_ratio} to an intermediate Sobolev space $H^s([0,T], \mathbb{R}^p)$.
\begin{cor}\label{cor:main_likelihood_ratio}
    Consider a finite-dimensional parameter space $X \subset \mathbb{R}^n$ and a Borel measure $\nu$ on $X$.
    Let $\mathcal{K}$ be a positive-definite trace-class operator in the Hilbert space $L^2([0,T], \mathbb{R}^p)$, with unbounded inverse $\mathcal{K}^{-1}$.
    Suppose that there exists a Sobolev space $H^s([0,T], \mathbb{R}^p)$ for some $s \in \mathbb{Z}^{+}$ such that:
    \begin{enumerate}
        \item The map $\mathcal{A}: X \rightarrow H^s([0,T], \mathbb{R}^p)$ is continuous;
        \item There exists a constant $C > 0$ such that \begin{align}\label{eqn:SobolevBound}
                  \norm{\mathcal{K}^{-1} \mathbf{h}}_{L^2([0,T], \mathbb{R}^p)} \leq C \norm{\mathbf{h}}_{H^s([0,T], \mathbb{R}^p)} \quad \text{for all $\mathbf{h} \in H^s([0,T], \mathbb{R}^p)$}
              \end{align}
              where $\norm{\cdot}_{H^s([0,T], \mathbb{R}^p)}$ denotes the corresponding Sobolev norm;
    \end{enumerate}
    Then for all $x\in X$, the measures
    \begin{enumerate}
        \item $\gamma|_x$, the Gaussian distribution on $L^2([0,T], \mathbb{R}^p)$ with mean $\mathcal{A}(x)$ and covariance $\mathcal{K}$, and
        \item $\bar{\gamma}$, the marginal distribution on $L^2([0,T], \mathbb{R}^p)$ obtained from marginalizing $x$ in the joint distribution in Eq.~\eqref{def:jointmeasure1},
    \end{enumerate}
    are absolutely continuous with respect to each other, with Radon--Nikodym density at $\mathbf{y}(t)$
    \begin{align}
        \frac{d\gamma|_x }{d\bar{\gamma} }(\mathbf{y}(t)) & = C(\mathbf{y}(t)) \exp\left(
        \frac{1}{T}\int_0^T \left(\mathbf{y}^\top(t)\left(\mathcal{K}^{-1}\mathcal{A}(x)\right)(t) - \frac{1}{2} \mathcal{A}(x)^\top(t)\left(\mathcal{K}^{-1}\mathcal{A}(x)\right)(t)\right) \ dt
        \right),
    \end{align}
    where $C(\mathbf{y}(t))$ is a constant independent of $x$.
\end{cor}
\begin{proof}
    The condition in Eq.~\eqref{eqn:SobolevBound} is equivalent to the continuity of the embedding $\iota: H^s([0,T], \mathbb{R}^p) \hookrightarrow \mathcal{H}_{\mathcal{K}^2}$.
    Then $\mathcal{A}: X \rightarrow H^s([0,T], \mathbb{R}^p) \hookrightarrow \mathcal{H}_{\mathcal{K}^2}$ is continuous, and the assumptions of Proposition~\ref{prop:main_likelihood_ratio} are met.
\end{proof}

Finally, we construct a covariance operator $\mathcal{K}$ for scalar-valued functions, satisfying condition 2 of Corollary~\ref{cor:main_likelihood_ratio}, as a base case that is generalizable to vector-valued functions such as our observation trajectories in Sec.~\ref{sec:measurementOfObservability}.
\begin{rem}\label{rem:main_likelihood_legendre}
    Let ${\left\{e_n(t) \right\}}_{n=0}^\infty$ be the orthonormal Legendre polynomials on $[0, T]$ under the inner product in Eq.~\eqref{eqn:InnerProductDef}.
    For a given function $h \in L^2([0,T])$, let $h_n = \langle h, e_n\rangle_{L^2([0,T])}$ be its $n$-th Fourier--Legendre coefficient.
    From~\cite{canuto2007spectral}, inequality (5.4.12), it follows that:
    \begin{align}\label{eqn:rapid_decay_of_legendre_coef}
        h \in H^s([0,T]) \Rightarrow \lvert h_n \rvert  \leq A n^{-s} \norm{h}_{H^s([0,T])},
    \end{align}
    where $A$ is a constant independent of $h$ and $n$.
    In particular, if $h(t)$ is smooth, then $h_n = {O}(n^{-s})$ for all $s > 0$.

    Now consider the operator $\mathcal{K}$, defined spectrally as the operator with eigenbasis ${\left\{e_n(t) \right\}}_{n=0}^\infty$ and eigenvalues ${\left\{\lambda_n \right\}}_{n=0}^\infty$, so that
    \begin{enumerate}
        \item $\lambda_n > 0$ for all $n$,
        \item $\tr(\mathcal{K}) = \sum_{n=0}^{\infty} \lambda_n < \infty$,
        \item For some $a > 0$, $\alpha > 1$ and $N > 0$, we have, for all $n > N$, that $\lambda_n \geq a n^{-\alpha}$,
    \end{enumerate}
    Then $\mathcal{K}$ is a valid covariance operator, and moreover is positive definite.
    We also have:
    \begin{align}\label{eqn:explicit_K_bound_1}
        \norm{\mathcal{K}^{-1}h}_{L^2([0,T])} \leq \max(\lambda_0, \ldots, \lambda_N) \norm{h}_{L^2([0,T])} + \frac{1}{a} \sum_{n = N+1}^{\infty} n^\alpha \lvert h_n \rvert,
    \end{align}
    whenever $\mathcal{K}^{-1}h \in L^2([0,T])$.
    Comparing Eqs.~\eqref{eqn:rapid_decay_of_legendre_coef} and~\eqref{eqn:explicit_K_bound_1} shows that, for any integer $s > \alpha + 1$, the condition $h \in \norm{h}_{H^s([0,T])}$ is sufficient for the right-hand side of Eq.~\eqref{eqn:explicit_K_bound_1} to converge, and in such a case an inequality of the form Eq.~\eqref{eqn:SobolevBound} holds.
\end{rem}

\section{Derivation of the asymptotic size of the minimum eigenvalue of the STLOG}\label{app:asymptotics}

This appendix provides the derivation of Proposition~\ref{prop:STLOGasymptotics}.
This asymptotic result relates the STLOG in Eq.~\eqref{eq:stlogExpr} to the local observability index~\(r_*\)~\eqref{eq:rstardef} defined as a bridge to rank-testing observability in Sec.~\ref{sec:observabilityAnalysis}.
We consider a fixed operating point of the system \((\mathbf{x}, \mathbf{u})\), and adopt several inessential simplifications.
Specifically, we set $\boldsymbol{\Sigma} = \mathbf{1}$, and simply denote the order-$r$ STLOG from Eq.~\eqref{eq:stlogExpr} by $\mathbf{W}^{(r)}$ without the subscript $\boldsymbol{\Sigma}^{-1}$.

We first state the main theorem of this appendix giving bounds on the minimum eigenvalue of the STLOG, then we give our proposition that bounds on this minimum eigenvalue influence local observability. Finally, we derive the relevant bounds.
\begin{thm}\label{thm:STLOGbound}
    Let~\((\mathbf{x},\mathbf{u})\) be given and~\(r_*(\mathbf{x},\mathbf{u})\) be its local observability index.
    Consider the minimum eigenvalue of the order-$r$ STLOG~\(\mathbf{W}^{(r)}(\mathbf{x},\mathbf{u};\Delta T)\).
    \begin{enumerate}
        \item\label{it:part1} If $r < r_*(\mathbf{x},\mathbf{u})$, then $\lambda_{\min}(\mathbf{W}^{(r)}(\mathbf{x},\mathbf{u};\Delta T)) = 0$ .
        \item\label{it:part2} If $r \geq r_*(\mathbf{x},\mathbf{u})$, then for all sufficiently small $\Delta T$, there exists positive constants $C_1(\mathbf{x},\mathbf{u}), C_2(\mathbf{x},\mathbf{u})$ independent of $\Delta T$ such that
        \begin{equation}\label{eq:STLOGublb}
            C_1(\mathbf{x},\mathbf{u})\Delta T^{2r_*(\mathbf{x},\mathbf{u}) + 1} \leq \lambda_{\min}(\mathbf{W}^{(r)}(\mathbf{x},\mathbf{u};\Delta T)) \leq C_2(\mathbf{x},\mathbf{u})\Delta T^{2r_*(\mathbf{x},\mathbf{u}) + 1}.
        \end{equation}
        In particular, there exists some $C \in [C_1 , C_2]$ and a sequence $\Delta T_k \to 0$ such that \(\lim_{k \to \infty} \frac{\lambda_{\min}(\mathbf{W}^{(r)}(\mathbf{x},\mathbf{u};\Delta T_k))}{\Delta T_k^{2r_*(\mathbf{x},\mathbf{u}) + 1}} = C\).
    \end{enumerate}
    \begin{proof}
        We prove Part~\ref{it:part1} in Proposition~\ref{prop:theoremPart1}, and derive $C_1(\mathbf{x},\mathbf{u}), C_2(\mathbf{x},\mathbf{u})$ in Propositions~\ref{prop:STLOGUBdone} and~\ref{prop:STLOGLBdone} respectively.
    \end{proof}
\end{thm}

While Theorem~\ref{thm:STLOGbound} gave bounds on \(\lambda_{\min}(\mathbf{W}^{(r)})\), we separately prove how such bounds can influence local weak observability of a system.
Recall in Sec.~\ref{sec:observabilityAnalysis} we stated that the system is \emph{locally weakly observable} at order $r$ if~\(\text{rank }\boldsymbol{\mathcal{O}}^{(r)} = \dim \mathcal{X}\). We introduce the following subspace:
\begin{align}
    \mathcal{H}_k = \bigcap_{i=0}^k \ker \left(DL_\mathbf{f}^i \mathbf{h}\right) = \ker \boldsymbol{\mathcal{O}}^{(k)}.
\end{align}
Since $\mathcal{H}_{k-1} \supset \mathcal{H}_k$, this decreasing sequence of subspaces reaches $\{0\}$ precisely at the local observability index~$r_*$.
This definition leads to the following Lemma:
\begin{lem}\label{lem:Hrspace}
    \
    \begin{enumerate}
        \item $\mathcal{H}_{r}=\{0\}$ $\Leftrightarrow$ the system is locally observable at order~$r$.
        \item The local observability index $r_*(\mathbf{x},\mathbf{u})$ is the smallest value of $r$ such that $\mathcal{H}_{r} = \{0\}$.
        \item In particular, if $r_* > 0$, then $\mathcal{H}_{r_* - 1}$ is nonzero.
    \end{enumerate}
\end{lem}

Before incorporating the STLOG into the derivation, we rewrite it in terms of \(\boldsymbol{\mathcal{O}}^{(r)}\).
For any $\mathbf{v}, \tilde{\mathbf{v}} \in \mathcal{X}$:\footnote{In Eq.~\eqref{eq:STLOGalt}, if instead of the standard inner product on $\mathcal{Y}$ we use the inner product induced by $\boldsymbol{\Sigma}^{-1}$, we recover the expressions in Eqs.~\eqref{eq:gramiansigmadef} and~\eqref{eq:stlogExpr}.}
\begin{align}\label{eq:STLOGalt}
    \mathbf{v}^\top \mathbf{W}^{(r)} \tilde{\mathbf{v}} = \int_0^{\Delta T} & {\left( \begin{bmatrix} \mathbf{1} & t \mathbf{1} & \cdots & \frac{t^r}{r!} \mathbf{1} \end{bmatrix} {\boldsymbol{\mathcal{O}}^{(r)}} \mathbf{v}\right)}^\top \left( \begin{bmatrix} \mathbf{1} & t \mathbf{1} & \cdots & \frac{t^r}{r!} \mathbf{1} \end{bmatrix} {\boldsymbol{\mathcal{O}}^{(r)}} \tilde{\mathbf{v}} \right) \ dt,
\end{align}
making $\mathbf{W}^{(r)}$ positive semi-definite by construction.
\begin{prop}\label{prop:theoremPart1}
    The system is not locally observable at order $r$ if and only if $\lambda_{\min}(\mathbf{W}^{(r)} ) = 0$.
    \begin{proof}
        Suppose the system is not locally observable at order $r$.
        By Lemma~\ref{lem:Hrspace}, there exists some nonzero $\mathbf{v} \in \mathcal{H}_{r}$.
        In particular, ${\boldsymbol{\mathcal{O}}^{(r)}} \mathbf{v} = 0$.
        Substituting such \(\mathbf{v}\) into Eq.~\eqref{eq:STLOGalt} gives $\mathbf{v}^\top\mathbf{W}^{(r)} \mathbf{v} = 0$, so $\lambda_{\min}(\mathbf{W}^{(r)}) = 0$.

        Conversely, if $\lambda_{\min}(\mathbf{W}^{(r)}) = 0$, then there exists some nonzero $\mathbf{v} \in \mathcal{X}$ such that $\mathbf{v}^\top\mathbf{W}^{(r)} \mathbf{v} = 0$.
        Rewriting Eq.~\eqref{eq:STLOGalt}:
        \begin{align}
            \int_0^{\Delta T} \norm{\begin{bmatrix} \mathbf{1} & t \mathbf{1} & \cdots & \frac{t^r}{r!} \mathbf{1} \end{bmatrix} {\boldsymbol{\mathcal{O}}^{(r)}} {\mathbf{v}}}^2 \ dt = 0,
        \end{align}
        from which we deduce \(
        \begin{bmatrix} \mathbf{1} & t \mathbf{1} & \cdots & \frac{t^r}{r!} \mathbf{1} \end{bmatrix} {\boldsymbol{\mathcal{O}}^{(r)}} {\mathbf{v}} = 0\) for all \(t \in [0, \Delta T]\), which implies ${\boldsymbol{\mathcal{O}}^{(r)}} {\mathbf{v}} = 0$, thus we conclude the system is not locally observable at order $r$.
    \end{proof}
\end{prop}

We first derive the upper bound in Eq.~\eqref{thm:STLOGbound}.
Let $r_*$ be the local observability index.
We only consider $\mathbf{W}^{(r)}$ with $r \geq r_*$ in the following, and assume $r_* > 0$ without loss of generality.

We use the identity on analytical minima of a quadratic form over a unit sphere \(\lambda_{\min}(\mathbf{W}^{(r)}) = \min_{\mathbf{v} \in \mathcal{X}, \ \norm{\mathbf{v}}= 1} \mathbf{v}^\top\mathbf{W}^{(r)} \mathbf{v}\). We thus seek an appropriate unit vector $\mathbf{v}$ that demonstrates the upper bound in Eq.~\eqref{thm:STLOGbound}.
Any unit vector~$\mathbf{v} \in \mathcal{H}_{r_* - 1}$ satisfies
\begin{align}\label{eq:ubcond1}
    \left(DL_\mathbf{f}^{r_*}\mathbf{h}\right) \mathbf{v} \neq 0, \  \left(DL_\mathbf{f}^{k}\mathbf{h} \right) \mathbf{v} = 0 \ \text{for all } k < r_*.
\end{align}
In particular,~$\norm{(DL_\mathbf{f}^{r_*}\mathbf{h}) \mathbf{v}} = \norm{{\boldsymbol{\mathcal{O}}^{(r_*)}} {\mathbf{v}}}$.
We pick a unit vector~$\mathbf{v}$ that minimizes~$\norm{(DL_\mathbf{f}^{r_*}\mathbf{h}) \mathbf{v}}$ and obtain:
\begin{align}\label{eq:ubcond2}
    \norm{(DL_\mathbf{f}^{r_*}\mathbf{h}) \mathbf{v}} = \sigma_{\min}\left({\boldsymbol{\mathcal{O}}^{(r_*)}}\vert_{\mathcal{H}_{r_* - 1}}\right) > 0,
\end{align}
where~$\vert_{\mathcal{H}_{r_* - 1}}$ denotes the restriction to the subspace~$\mathcal{H}_{r_* - 1}$.
Now we evaluate $\mathbf{v}^\top\mathbf{W}^{(r)} \mathbf{v}$ using our selected $\mathbf{v}$:
\begin{equation}\label{eq:ubcalc1}
    \begin{bmatrix} \mathbf{1} & t \mathbf{1} & \cdots & \frac{t^r}{r!} \mathbf{1} \end{bmatrix} {\boldsymbol{\mathcal{O}}^{(r)}} {\mathbf{v}} = \frac{t^{r_*}}{r_*!} \left(DL_\mathbf{f}^{r_*}\mathbf{h}\right) \mathbf{v} + t^{r_*+ 1}\mathbf{a}(\mathbf{v}, t),
\end{equation}
for some polynomial $\mathbf{a}(\mathbf{v}, t)$.
Substituting Eq.~\eqref{eq:ubcalc1} into Eq.~\eqref{eq:STLOGalt} and evaluating the integral gives
\begin{equation}
    \mathbf{v}^\top\mathbf{W}^{(r)} \mathbf{v} = \frac{\Delta T^{2r_*+1}\norm{\left(DL_\mathbf{f}^{r_*}\mathbf{h}\right) \mathbf{v}}^2}{{(r_*!)}^2(2r_*+1)} + \Delta T^{2r_*+2}\alpha(\mathbf{v}, \Delta T),
\end{equation}
for some polynomial $\alpha(\mathbf{v}, \Delta T)$.
Let $\abs{\alpha(\mathbf{v}, \Delta T)} \leq A$ for all sufficiently small $\Delta t$ and all unit vectors $\mathbf{v}$, and $B = \frac{\sigma^2_{\min}\left({\boldsymbol{\mathcal{O}}^{(r_*)}}\vert_{\mathcal{H}_{r_* - 1}}\right) }{{{(r_*!)}^2}(2r_*+1)}$.
We have:
\begin{align}
    \lambda_{\min}(\mathbf{W}^{(r)}) \leq \mathbf{v}^\top\mathbf{W}^{(r)} \mathbf{v} \leq \Delta T^{2r_*+1}\left(B + A\Delta T\right).
\end{align}
For any $\epsilon >0$, choosing $\Delta T < \epsilon B/A$ gives:
\begin{prop}\label{prop:STLOGUBdone}
    For all $\epsilon > 0$, there exists $\delta > 0$ such that
    \begin{align}\label{eq:STLOGUBdone}
        \lambda_{\min}(\mathbf{W}^{(r)}) \leq \frac{(1+\epsilon) \Delta T^{2r_*+1} \sigma^2_{\min}\left({\boldsymbol{\mathcal{O}}^{(r_*)}}\vert_{\mathcal{H}_{r_* - 1}}\right) }{{(r_*!)}^2(2r_*+1)},
    \end{align}
    for all $\Delta T < \delta $.
\end{prop}

Next, we derive the lower bound in Eq.~\eqref{eq:STLOGublb}.
We rewrite the STLOG into block matrix form
\begin{align}\label{eq:STLOGaltHilb}
    \mathbf{W}^{(r)} = \Delta T {\boldsymbol{\mathcal{O}}^{(r)}}^\top \boldsymbol{\Lambda}^{(r)} \boldsymbol{\mathcal{H}}^{(r)} \boldsymbol{\Lambda}^{(r)} {\boldsymbol{\mathcal{O}}^{(r)}},
\end{align}
where $\boldsymbol{\Lambda}^{(r)}$ is the block diagonal matrix
\begin{equation}\label{eq:Lambdadef}
    \boldsymbol{\Lambda}^{(r)} =
    \diag\begin{pmatrix}
        \mathbf{1} & \Delta T\mathbf{1} & \cdots & \frac{\Delta T^r}{r!} \mathbf{1}
    \end{pmatrix},
\end{equation}
and $\boldsymbol{\mathcal{H}}^{(r)}$ is the block matrix whose entries correspond to a finite-order Hilbert matrix~\cite{todd1961computational}:
\begin{equation}\label{eq:Hilbertmatdef}
    \boldsymbol{\mathcal{H}}^{(r)} =
    \begin{bmatrix}
        \mathbf{1}            & \frac{1}{2}\mathbf{1}   & \cdots & \frac{1}{r}\mathbf{1}    \\
        \frac{1}{2}\mathbf{1} & \frac{1}{3}\mathbf{1}   & \cdots & \frac{1}{r+1}\mathbf{1}  \\
        \vdots                & \vdots                  & \ddots & \vdots                   \\
        \frac{1}{r}\mathbf{1} & \frac{1}{r+1}\mathbf{1} & \cdots & \frac{1}{2r+1}\mathbf{1}
    \end{bmatrix}.
\end{equation}
It is well-known~\cite{todd1961computational} that the Hilbert matrix $\boldsymbol{\mathcal{H}}^{(r)}$ for each $r$ is positive definite.
This allows us to prove the lower bound in~\eqref{eq:STLOGublb} with a lower bound on $\boldsymbol{\Lambda}^{(r)} {\boldsymbol{\mathcal{O}}^{(r)}} \mathbf{v}$.

\begin{lem}
    Let $\mathbf{v} \in \mathcal{X}$ be a unit vector. Then for $\Delta T \leq 1$,
    \begin{align}\label{eq:LOlb1}
        \norm{\boldsymbol{\Lambda}^{(k)} {\boldsymbol{\mathcal{O}}^{(k)}} \mathbf{v}} \geq \frac{\Delta T^k}{k!} \sigma_{\min}\left( {\boldsymbol{\mathcal{O}}^{(k)}} \right).
    \end{align}
\end{lem}
\begin{proof}
    By substituting the value of $\sigma_{\min}\left(\boldsymbol{\Lambda}^{(r)}\right)$ from~\eqref{eq:Lambdadef} and $\norm{\mathbf{v}} = 1$ in the last step, we get
    \begin{equation*}
        \norm{\boldsymbol{\Lambda}^{(k)} {\boldsymbol{\mathcal{O}}^{(k)}} \mathbf{v}} \geq \sigma_{\min}\left( \boldsymbol{\Lambda}^{(k)} {\boldsymbol{\mathcal{O}}^{(k)}} \right)\norm{\mathbf{v}}  \geq \sigma_{\min}\left( \boldsymbol{\Lambda}^{(k)}\right) \sigma_{\min}\left({\boldsymbol{\mathcal{O}}^{(k)}} \right)\norm{\mathbf{v}} = \frac{\Delta T^k}{k!} \sigma_{\min}\left( {\boldsymbol{\mathcal{O}}^{(k)}} \right),
    \end{equation*}
\end{proof}
The bound~\eqref{eq:LOlb1} admits a straightforward but important improvement for $r \geq r_*$:
\begin{lem}
    Let $\mathbf{v} \in \mathcal{X}$ be a unit vector, and let $r \geq r_*$. Then for $\Delta T \leq 1$,
    \begin{equation}\label{eq:LOlb2}
        \norm{\boldsymbol{\Lambda}^{(r)} {\boldsymbol{\mathcal{O}}^{(r)}} \mathbf{v} } \geq \frac{\Delta T^{r_*}}{r_*!} \sigma_{\min}\left( {\boldsymbol{\mathcal{O}}^{(r_*)}} \right),
    \end{equation}
    where $\sigma_{\min}({\boldsymbol{\mathcal{O}}^{(r_*)}})$ is positive and independent of $\Delta T$.
\end{lem}
\begin{proof}
    Write the vector $\boldsymbol{\Lambda}^{(r)} {\boldsymbol{\mathcal{O}}^{(r)}} \mathbf{v}$ as a block of two vectors
    \[
        \boldsymbol{\Lambda}^{(r)} {\boldsymbol{\mathcal{O}}^{(r)}} \mathbf{v} =
        \begin{bmatrix}
            \boldsymbol{\Lambda}^{(r_*)} {\boldsymbol{\mathcal{O}}^{(r_*)}} \mathbf{v} \\
            \mathbf{w}
        \end{bmatrix},
    \]
    where $\mathbf{w}$ is a $(r-r_*)\dim\mathcal{Y} \times 1$ column vector.
    Now
    \[
        \norm{\boldsymbol{\Lambda}^{(r)} {\boldsymbol{\mathcal{O}}^{(r)}} \mathbf{v}} = \sqrt{\norm{\boldsymbol{\Lambda}^{(r_*)} {\boldsymbol{\mathcal{O}}^{(r_*)}} \mathbf{v}}^2 + \norm{\mathbf{w}}^2} \geq \norm{\boldsymbol{\Lambda}^{(r_*)} {\boldsymbol{\mathcal{O}}^{(r_*)}} \mathbf{v}} \geq \frac{\Delta T^{r_*}}{r_*!} \sigma_{\min}\left( {\boldsymbol{\mathcal{O}}^{(r_*)}} \right),
    \]
    where we used the bound~\eqref{eq:LOlb1} for $k=r_*$ in the last step.
\end{proof}
\begin{prop}\label{prop:STLOGLBdone}
    For all $\Delta T \leq 1$, we have
    \begin{equation}\label{eq:STLOGLBdone}
        \lambda_{\min}(\mathbf{W}^{(r)}) \geq \frac{a^{(r)} \Delta T^{2r_* + 1} \sigma^2_{\min}\left( {\boldsymbol{\mathcal{O}}^{(r_*)}} \right)}{{(r_*!)}^2},
    \end{equation}
    where $a^{(r)} > 0$ is the minimum eigenvalue of the Hilbert matrix of order $r$~\eqref{eq:Hilbertmatdef}.
\end{prop}
\begin{proof}
    Let $\mathbf{v}$ be a unit vector in $\mathcal{X}$. Using~\eqref{eq:STLOGaltHilb}, we apply the factorization \(\mathbf{v}^\top \mathbf{W}^{(r)} \mathbf{v} = \Delta T\langle \boldsymbol{\mathcal{H}}^{(r)} \boldsymbol{\Lambda}^{(r)} {\boldsymbol{\mathcal{O}}^{(r)}} \mathbf{v},  \boldsymbol{\Lambda}^{(r)} {\boldsymbol{\mathcal{O}}^{(r)}} \mathbf{v}\rangle,\) where $\boldsymbol{\mathcal{H}}^{(r)}$ is the Hilbert matrix~\eqref{eq:Hilbertmatdef}.
    Therefore,
    \begin{equation*}
        \mathbf{v}^\top \mathbf{W}^{(r)} \mathbf{v} \geq \Delta T a^{(r)} \norm{\boldsymbol{\Lambda}^{(r)} {\boldsymbol{\mathcal{O}}^{(r)}} \mathbf{v} }^2 \geq \Delta T a^{(r)} {\left( \frac{\Delta T^{r_*}}{r_*!} \sigma_{\min}\left( {\boldsymbol{\mathcal{O}}^{(r_*)}} \right) \right)}^2,
    \end{equation*}
    using~\eqref{eq:LOlb2} in the last step.
    Simplifying and minimizing over unit vectors $\mathbf{v} \in \mathcal{X}$ gives the inequality~\eqref{eq:STLOGLBdone}.
\end{proof}
Combining Propositions~\ref{prop:STLOGUBdone},~\ref{prop:STLOGLBdone}, we have, for all sufficiently small $\Delta T$,
\begin{align}\label{eq:STLOGboundOr}
    \alpha  \sigma^2_{\min}\left( {\boldsymbol{\mathcal{O}}^{(r_*)}} \right) \leq \frac{\lambda_{\min}(\mathbf{W}^{(r)})}{\Delta T^{2r_* + 1}} \leq \beta  \sigma^2_{\min}\left({\boldsymbol{\mathcal{O}}^{(r_*)}}\vert_{\mathcal{H}_{r_* - 1}}\right),
\end{align}
where $\alpha, \beta$ are constants.
Thus, we proved the inequality~\eqref{eq:STLOGublb} in Theorem~\ref{thm:STLOGbound} by explicitly constructing $C_1(\mathbf{x},\mathbf{u}), C_2(\mathbf{x},\mathbf{u})$ in terms of $\boldsymbol{\mathcal{O}}^{(r_*)}$ only.
The last part of Theorem~\ref{thm:STLOGbound} follows from the Bolzano--Weierstrass theorem.

\section{Derivation of the choice of STLOG order for quadrotor CLS}\label{app:choiceOfR}

This appendix provides the derivation that justifies the choice of STLOG order \(r = 5\) for the quadrotor CLS.\@

\begin{prop}\label{prop:UTdynsys1}
    Consider a dynamical system in Eqs.~\eqref{eq:dynamicalSystem} and~\eqref{eq:observation}, whose state space can be partitioned into $\mathcal{X} = \mathcal{Z} \times \mathcal{Q}$, with state $\mathbf{x} = (\mathbf{z}, \mathbf{q})$, \(\mathbf{z}\in \mathcal{Z}\), \(\mathbf{q}\in \mathcal{Q}\), such that the dynamics and the observations can be decomposed into:
    \begin{align}\label{eq:UTdynsys1}
        \begin{bmatrix}
            \dot{\mathbf{z}} \\
            \dot{\mathbf{q}}
        \end{bmatrix} & = \begin{bmatrix}
                              \mathbf{f}_{z}(\mathbf{z}, \mathbf{q}, \mathbf{u}) \\
                              \mathbf{f}_{q}(\mathbf{q}, \mathbf{u})
                          \end{bmatrix}, \quad \mathbf{y} = \mathbf{h}(\mathbf{z}, \mathbf{q}) =
        \begin{bmatrix}
            \mathbf{d}(\mathbf{z}, \mathbf{q}) \\ \mathbf{a}(\mathbf{q})
        \end{bmatrix},
    \end{align}
    where $\mathbf{f}_{z}$ and $\mathbf{f}_{q}$ are components of the dynamics $\mathbf{f} = (\mathbf{f}_z, \mathbf{f}_q)$, and $\mathbf{d}: \mathcal{Z} \times \mathcal{Q} \rightarrow \mathbb{R}^s$ and $\mathbf{a}: \mathcal{Q} \rightarrow \mathbb{R}^{p-s}$ are components for the observation map $\mathbf{h}: \mathcal{X} \rightarrow\mathcal{Y} = \mathbb{R}^p$.
    For all $(\mathbf{x}, \mathbf{u})$, the local observability index $r_*$ in Eq.~\eqref{eq:rstardef} for the system in Eq.~\eqref{eq:UTdynsys1} satisfies $s(r_*(\mathbf{x}, \mathbf{u}) + 1) \geq \dim(\mathcal{Z})$ .
\end{prop}

\begin{proof}
    The iterated Lie derivatives of \(\mathbf{h}\) along \(\mathbf{f}\) can be evaluated inductively as
    \begin{equation}
        L_\mathbf{f} \mathbf{h} (\mathbf{z},\mathbf{q}) = L_\mathbf{f}(\mathbf{d}(\mathbf{z},\mathbf{q}), \mathbf{a}(\mathbf{q})) = \left(L_{\mathbf{f}_{z}} + L_{\mathbf{f}_{q}}\right)\begin{bmatrix}
            \mathbf{d}(\mathbf{z}, \mathbf{q}) \\ \mathbf{a}(\mathbf{q})
        \end{bmatrix} =       \begin{bmatrix}
            \mathbf{d}^{(1)}(\mathbf{z}, \mathbf{q}) \\ \mathbf{a}^{(1)}(\mathbf{q})
        \end{bmatrix},\quad L^i_\mathbf{f} \mathbf{h} (\mathbf{z},\mathbf{q}) =
        \begin{bmatrix}
            \mathbf{d}^{(i)}(\mathbf{z}, \mathbf{q}) \\ \mathbf{a}^{(i)}(\mathbf{q})
        \end{bmatrix},
    \end{equation}
    for some functions $\mathbf{d}^{(i)}$ and $\mathbf{a}^{(i)}$, where $\mathbf{a}^{(i)}$ has no $\mathbf{z}$-dependence.
    We next evaluate the gradients of the Lie derivatives:
    \begin{align}
        D L^i_\mathbf{f} \mathbf{h} = \begin{bmatrix}
                                          D_\mathbf{z} \mathbf{d}^{(i)} & D_\mathbf{q} \mathbf{d}^{(i)} \\
                                          D_\mathbf{z} \mathbf{a}^{(i)} & D_\mathbf{q} \mathbf{a}^{(i)} \\
                                      \end{bmatrix} = \begin{bmatrix}
                                                          D_\mathbf{z} \mathbf{d}^{(i)} & D_\mathbf{q} \mathbf{d}^{(i)} \\
                                                          \mathbf{0}                    & D_\mathbf{q} \mathbf{a}^{(i)} \\
                                                      \end{bmatrix}.
    \end{align}

    Finally, we build the observability matrix:
    \begin{align}
        \boldsymbol{\mathcal{O}}^{(r)} = \begin{bmatrix}
                                             D_\mathbf{z} \mathbf{d}       & D_\mathbf{q} \mathbf{d}       \\
                                             \mathbf{0}                    & D_\mathbf{q} \mathbf{a}       \\
                                             \vdots                        & \vdots                        \\
                                             D_\mathbf{z} \mathbf{d}^{(r)} & D_\mathbf{q} d^{(r)}          \\
                                             \mathbf{0}                    & D_\mathbf{q} \mathbf{a}^{(r)} \\
                                         \end{bmatrix}.
    \end{align}
    The submatrix of $\boldsymbol{\mathcal{O}}^{(r)}$ consisting of the first $\dim \mathcal{Z}$ columns has no more than $s(r+1)$ nonzero rows, so the row span of $\boldsymbol{\mathcal{O}}^{(r)}$ will always be deficient if $s(r+1) < \dim \mathcal{Z}$.

\end{proof}

\begin{cor}\label{cor:rstar_quadrotorCLS}
    For the quadrotor CLS system in Eqs.~\eqref{eq:dynamicsModel} and~\eqref{eq:observationModel}, $r_* \geq 5$.
\end{cor}
\begin{proof}
    Applying Proposition~\ref{prop:UTdynsys1} to our quadrotor CLS gives $1 \times (r_*+1) \geq \dim \mathcal{Z} = 6$ as required, since:
    \begin{enumerate}
        \item $\mathcal{Z}$ is the space for the translational state, including position and velocity $\mathbf{z} = (\mathbf{r}, \mathbf{v})$, giving $\dim(\mathcal{Z}) = 6$, and
        \item $\mathbf{d}$ is the $1$-dimensional range-only measurement, giving $s=1$.
    \end{enumerate}
    by the nature of how position and velocity are partially observed through range-only measurements.
\end{proof}

\end{appendices}

\section*{Acknowledgements}%
\label{sec:acknowledgments}
We would like to acknowledge the sponsorship by the Natural Sciences and Engineering Research Council of Canada (NSERC) under the Discovery Grant RGPIN-2023-05148.

\bibliography{refs}

\end{document}